\documentclass{article}

\usepackage{microtype}
\usepackage{graphicx}
\usepackage{float,graphicx,epstopdf}
\usepackage{bbm}  

\usepackage{dsfont,amssymb,amsmath, graphicx,enumitem, hyperref, color} 
\usepackage{amsfonts,dsfont,mathtools, mathrsfs,amsthm,fancyhdr, caption, subcaption} 
\usepackage[amssymb, thickqspace]{SIunits}

\usepackage{algorithm, algorithmic}
\renewcommand{\algorithmicrequire}{\textbf{Input:}} 
\renewcommand{\algorithmicensure}{\textbf{Output:}}

\usepackage{multirow}

\makeatletter
\def\munderbar#1{\underline{\sbox\tw@{$#1$}\dp\tw@\z@\box\tw@}}
\makeatother

\newtheorem{theorem}{Theorem}[section]
\newtheorem{definition}[theorem]{Definition}
\newtheorem{lemma}[theorem]{Lemma}

\newtheorem{proposition}[theorem]{Proposition}

\newcommand{\be}{\begin{equation}}
\newcommand{\ee}{\end{equation}}
\newcommand{\bea}{\begin{equation*}\begin{aligned}}
\newcommand{\eea}{\end{aligned}\end{equation*}}
\newcommand{\ds}{\displaystyle}

\newcommand{\R}{\mathbb{R}}

\newcommand{\Min}{\min\limits_}
\newcommand{\Sup}{\sup\limits_}
\newcommand{\Inf}{\inf\limits_}
\newcommand{\Tr}[1]{\Trace \big[ #1 \big]}

\newcommand{\wh}{\widehat}
\newcommand{\mc}{\mathcal}
\newcommand{\mbb}{\mathbb}

\newcommand{\norm}[1]{\left\|#1\right\|}
\newcommand{\cov}{\Sigma} 
\newcommand{\covsa}{\wh{\Sigma}} 

\newcommand{\Pnom}{\wh{\mbb P}}
\newcommand{\QQ}{\mbb Q}

\DeclareMathOperator{\Trace}{Tr}

\DeclareMathOperator{\st}{s.t.}

\newcommand{\PSD}{\mathbb{S}_{+}} 
\newcommand{\PD}{\mathbb{S}_{++}} 
\newcommand{\Let}{\triangleq}
\newcommand{\opt}{^\star}
\newcommand{\eps}{\varepsilon}

\newcommand{\EE}{\mathds{E}}

\newcommand{\half}{\frac{1}{2}}
\newcommand{\dualvar}{\kappa}

\newcommand{\ie}{{\em i.e.}}

\newcommand{\Sym}{\mbb S}
\newcommand{\m}{\mu}
\newcommand{\msa}{\wh \m}

\usepackage{booktabs} 
\usepackage[accepted]{icml2021}
\usepackage{hyperref}
\usepackage{xcolor}

\usepackage{xr-hyper}
\makeatletter
\newcommand*{\addFileDependency}[1]{
  \typeout{(#1)}
  \@addtofilelist{#1}
  \IfFileExists{#1}{}{\typeout{No file #1.}}
}
\makeatother
\newcommand*{\myexternaldocument}[1]{%
    \externaldocument{#1}%
    \addFileDependency{#1.tex}%
    \addFileDependency{#1.aux}%
}
\myexternaldocument{appendix}
\icmltitlerunning{}

\begin{document}

\twocolumn[
\icmltitle{Sequential Domain Adaptation \\ by Synthesizing Distributionally Robust Experts}



\begin{icmlauthorlist}
\icmlauthor{Bahar Taskesen}{epfl}
\icmlauthor{Man-Chung Yue}{hong}
\icmlauthor{Jos\'{e} Blanchet}{su}
\icmlauthor{Daniel Kuhn}{epfl}
\icmlauthor{Viet Anh Nguyen}{su,vin}
\end{icmlauthorlist}

\icmlaffiliation{hong}{Department of Applied Mathematics, The Hong Kong Polytechnic University}
\icmlaffiliation{epfl}{Risk Analytics and Optimization Chair, Ecole Polytechnique F\'{e}d\'{e}rale de Lausanne}
\icmlaffiliation{su}{Department of Management Science and Engineering, Stanford University}
\icmlaffiliation{vin}{VinAI Research, Vietnam}
\icmlcorrespondingauthor{Bahar Taskesen}{bahar.taskesen@epfl.ch}

\icmlkeywords{Machine Learning, ICML}

\vskip 0.3in]
\printAffiliationsAndNotice{} 
\begin{abstract}
    
    Least squares estimators, when trained on a few target domain samples, may predict poorly. Supervised domain adaptation aims to improve the predictive accuracy by exploiting additional labeled training samples from a source distribution that is close to the target distribution. Given available data, we investigate novel strategies to synthesize a family of least squares estimator experts that are robust with regard to moment conditions. When these moment conditions are specified using Kullback-Leibler or Wasserstein-type divergences, we can find the robust estimators efficiently using convex optimization. We use the Bernstein online aggregation algorithm on the proposed family of robust experts to generate predictions for the sequential stream of target test samples. Numerical experiments on real data show that the robust strategies may outperform non-robust interpolations of the empirical least squares estimators.

\end{abstract}
\section{Introduction}
\label{sect:intro}
\begin{figure*}
    \centering
    \includegraphics[width=\textwidth]{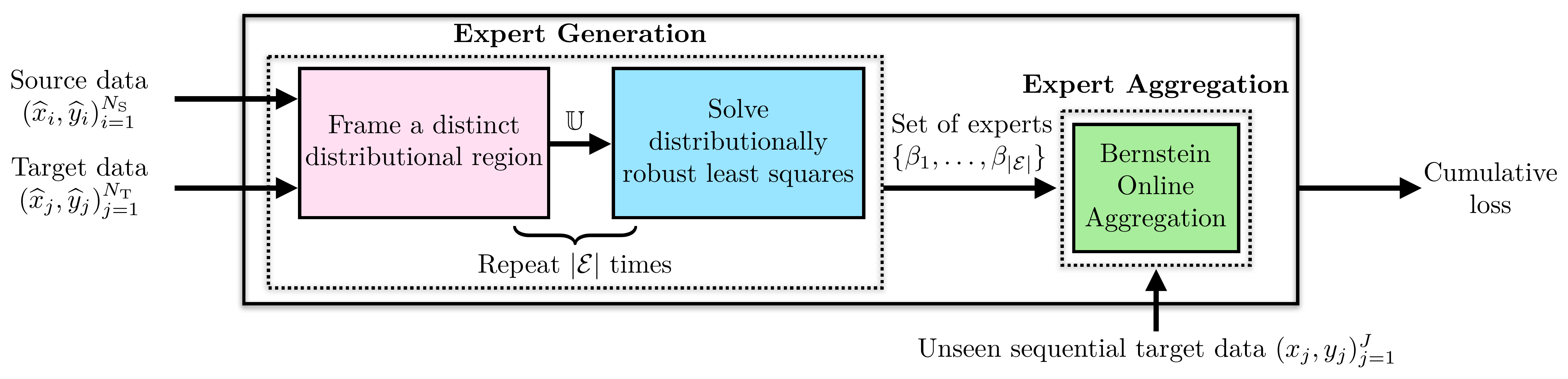}
    \vspace{-.8cm}
    \caption{The architecture of our framework for supervised domain adaptation when the unseen target test samples arrive sequentially.}
    \vspace{-5mm}
    \label{fig:method}
\end{figure*}


    
    A natural approach to improving predictive performance in data-scarce tasks involves translating informative signals from a data-abundant source domain to the data-scarce target domain. This transfer of knowledge is commonly referred to as domain adaptation or transfer learning, and it is increasingly applied in a wide range of settings, see for example~\citet{ref:wilson2020survey, ref:chu-wang-2018-survey, ref:weiss2016survey} and~\citet{ref:redko2019advances}. 
    We consider the supervised domain adaptation setting with scarce labeled target data. The key challenge here is the absence of meaningful data to tune any parameters. However, in many practically relevant applications, new data will arrive sequentially to enrich the information on the target domain. In this case, many online algorithms can be utilized to adaptively learn the best predictor on the target domain, which also guarantee optimal asymptotic regrets~\cite{ref:lattimore2020bandit}. 

    In this paper, we take a pragmatic approach to resolve a specific setup of the domain adaptation problem. We assume access to a scarce labelled target data, and the future target data arrives sequentially.
    For example, consider understanding the dynamics of ride-sharing platforms requires insights about the demand and supply from both sides of the market. These insights are signalled through the ride fares, which can be explained by characteristics such as the travel distances and the origin-destination pairs of the trips, the time of the day as well as the weather conditions. The capability to correctly predict ride fares directly translates into improved profit forecasts, and thus it vitally supports the growth of new-coming platforms. In a competitive market, a follower (e.g., Lyft) needs to target a slightly different market segment than the leader (e.g., Uber) who had entered earlier. Thus, the demand and supply characteristics for the follower may differ from those of the leader. Nevertheless, as both platforms provide on-demand transportation, it is reasonable to assume that their supply and demand dynamics are similar. The follower, who possesses limited data, can query demand on the leader’s platform to collect data in order to leap forward in its predictive precision.
    Our approach to solve this problem is illustrated in Figure~\ref{fig:method} and it consists of two components:
    \begin{enumerate}[leftmargin =4mm, itemsep=0.5mm]
\vspace{-.3cm}
        \item \textbf{Expert Generation Module:} This module generates a set of competitive experts $\mc E$ by fine-tuning the explanatory power of the source domain data and harnessing the signal guidance from the scarce target domain data.
        \item \textbf{Expert Aggregation Module:} Acting on the sequential arrival of the unseen target data, this module aggregates the predictive capability of the generated experts via an online aggregation mechanism. In this work we will use the Bernstein Online Aggregation mechanism. 
    \end{enumerate}

 We will propose two ways to generate the experts. The first approach generates experts corresponding to optimal decisions along a path, with the intention to interpolate between the source and the target distributions. We will consider two types of trajectories, guided by either the Kullback-Leibler or the Wasserstein divergence. The second approach generates distribution regions around both the source and the target. The intersection of these regions is used to generate distributionally robust experts. The geometrical intuition is to find the ``direction" induced by the aforementioned divergences, in which the source data can explain the target data. Once the experts are deployed, the aggregation mechanism is executed without re-adapting the experts.
 
   Our ultimate goal is to ensure a competitive performance in the short term and not in the asymptotic regime when the number of test samples from the target domain tends to infinity. Indeed, as soon as the target sample size is sufficient, training the machine learning model on all available target data becomes more attractive. From a short term horizon benchmark, our approach offers an appealing \emph{warm start} for online training procedure, and it may also lead to a faster convergence rate depending on the underlying algorithm.

\textbf{Contributions.}  Our paper explores the expert generation problem in the context of supervised domain adaptation.

    \begin{itemize}[leftmargin = 2.5mm, itemsep=0.5mm]
\vspace{-.3cm}
        \item We introduce a novel framework to synthesize a family of robust least squares experts by altering various moment-based distribution sets. These sets gradually interpolate from the source information to the target information, capturing different belief levels on the explanatory power of the source domain onto the target domain. 
        
        \item We present two intuitive strategies to construct the sets of moment information, namely the ``Interpolate, then Robustify'' and the ``Surround, then Intersect'' strategies. Both strategies are simply characterized by two parameters representing the aforementioned explanatory power of belief of the source domain and the level of desired robustness.
        
        \item We show that when the moment information
        is prescribed using a Kullback-Leibler or a Wasserstein-type divergence, the experts are efficiently formed by solving convex optimization problems, that can even be solved by a first-order gradient descent algorithm or off-the-shelf solvers.
    \end{itemize}
    \vspace{-.3cm}

    This paper is structured as follows. Section~\ref{sec:problem} delineates the problem setup and describes in details two common strategies to generate experts: the convex combination and the reweighting strategies. Section~\ref{sect:DRLR} introduces our framework to generate experts, while Section~\ref{sec:IR} and~\ref{sec:SI} dive into details about our ``Interpolate, then Robustify'' and our ``Surround, then Intersect'' strategies, respectively. Section~\ref{sec:numerical} demonstrates experimentally that the proposed robust strategies systematically outperform non-robust interpolations of the empirical least squares estimators.
    
\textbf{Literature Review.}
 Domain adaptation arises in various applications including natural language processing~\citep{ref:sogaard2013semi, ref:li2012literature, ref:jiang-zhai-2007-instance, ref:blitzer2006domain}, survival analysis~\citep{ref:li2016transfer} and computer vision~\citep{ref:WANG2018135, Csurka2017}.
 Domain adaptation methods can be classified into three categories. 
 Unsupervised domain adaptation only requires unlabelled target data, but in large amounts~\citep{ref:ghifary2016deep, ref:baktashmotlagh2013unsupervised, ref:ganin2015unsupervised, ref:wang2020distributionally, ref:long2016unsupervised, ref:ben2007analysis, ref:courty2016optimal}. Semi-supervised domain adaptation requires labelled target data~\citep{ref:yao2015semi, ref:kumar2010co, ref:sindhwani2005co, ref:lopez2013semi,  ref:avishek2011active, ref:de2020adversarial, ref:sun2011two}.
 Finally, supervised domain adaptation only requires scarce labelled target data~\cite{ref:motiian2017unified, ref:motiian2017few, ref:tzeng2015simultaneous, ref:koniusz2017domain}.
 If the target data is scarce and label information is available, supervised domain adaptation outperforms unsupervised domain adaptation~\citep{ref:motiian2017unified}. 
 The domain adaptation literature further ramifies by imposing different distributional assumptions into covariate shift~\citep{ref:shimodaira2000improving, ref:sugiyama2008direct} or label shift~\citep{ref:lipton2018detecting, ref:azizzadenesheli2018regularized}.

 The domain adaptation literature for regression problems focuses primarily on instance-based reweighting strategies~\citep{ref:garcke2014importance, ref:sugiyama2008direct, ref:garcke2014importance, ref:huang2006correcting, ref:CORTES2014domain, ref:chen2016robust}, which aim to minimize some distance between the source and target distributions.
 Most of the instance-based methods solve an optimization problem to find the weights of the instances~\citep{ref:garcke2014importance, ref:cortes2019adaptation}, which may be computationally expensive when data is abundant.
 Other approaches rely on deep learning models to minimize the discrepancy between the domain distributions~\citep{ref:zhao2018adversarial, ref:richard2020unsupervised}.
 The literature on regression for domain adaptation also extends towards boosting-based methods~\citep{ref:Pardoe2010BoostingFR}, and 
 deep learning methods~\citep{ref:SALAKEN2019565}.

    Our paper also uses ideas and techniques from robust optimization and adversarial training, which have attracted considerable attention in machine learning~\citep{ref:namkoong2016stochastic, ref:gao2018robust, ref:blanchet2019robust,ref:nguyen2019calculating}. Robust optimization for least squares problem with uncertain data was studied in~\citet{ref:Ghaoui1997robust}. Distributionally robust optimization with moment ambiguity sets was proposed in~\citet{ref:delage2010distributionally} and extended in~\citet{ref:goh2010distributionally} and~\citet{ref:kuhn2019wasserstein}. Ambiguity sets prescribed by divergences were previously used to robustify Bayes classification~\cite{ref:nguyen2019optimistic,ref:nguyen2020robust}.

    Our work is also similar to~\citet{ref:chen2016robust} that consider unsupervised domain adaptation regression, and~\citet{ref:wang2020distributionally} that consider robust domain adaption for the classification setting.



\textbf{Notation.} We use $I_d$ to denotes the identity matrix in $\mathbb R^d$. The set of $p$-by-$p$ positive (semi-)definite matrices is denoted by $\PD^p$ ($\PSD^p$). All proofs are relegated to the Appendix.

\section{Problem Statement and Background}
\label{sec:problem}
We consider a generic linear regression setting, in which $X$ is a $d$-dimensional covariate and $Y$ is a univariate response variable. In the context of supervised domain adaptation, we have access to the source domain data $(\wh x_i, \wh y_i)_{i=1}^{N_{\rm S}}$ consisting of $N_{\rm S}$ labelled samples drawn from the source distribution. In addition, we are given a limited number of $N_{\rm T}$ labelled samples $(\wh x_j, \wh y_j)_{j=1}^{N_{\rm T}}$ from the target distribution. Our goal is to predict the responses of the test samples $(x_j, y_j)_{j=1}^J$, which are drawn from the target distribution and arrive sequentially. To this end, we will construct several experts.

In the linear regression setting, each expert is characterized by a vector $\beta \in \R^d$. Given a covariate-response pair $(x, y) \in \R^d \times \R$, we use the square loss function to measure the mismatch between the expert's prediction $\beta^\top x$ and the actual response $y$. Using the target domain data $(\wh x_i, \wh y_i)_{i=1}^{N_{\rm T}}$, one approach is to solve the ridge regression problem
\[
    \Min{\beta \in \R^d}~\frac{1}{N_{\rm T}} \sum_{j=1}^{N_{\rm T}} (\beta^\top \wh x_j -  \wh y_j)^2 + \eta \| \beta\|_2^2
\]
for some $\eta \ge 0$ to obtain the empirical target predictor
\[
    \wh \beta_{\rm T} = \left( \frac{1}{N_{\rm T}} \sum_{j=1}^{N_{\rm T}} \wh x_j \wh x_j^\top + \eta I_d \right)^{-1} \left( \frac{1}{N_{\rm T}} \sum_{j=1}^{N_{\rm T}} \wh x_j \wh y_j \right).
\]
When $N_{\rm T}$ is small, however, the empirical target predictor may perform poorly on the future target data $(x_j, y_j)_{j=1}^J$.

If the source domain distribution is sufficiently close to the target domain distribution, it is expedient to exploit the available information in the source domain data to construct better predictors for the target domain data. With this promise, one can synthesize several predictors to form an ensemble of experts, and one can apply an online aggregation scheme to predict on the unseen target data. We now first describe several interpolation schemes to generate experts.

\textbf{Convex Combination Strategy. }
Denote by $\wh \beta_{\rm S}$ the empirical source predictor, which is obtained by solving the ridge regression problem on the source data. The convex combination strategy generates predictors by forming convex combinations between $\wh \beta_{\rm S}$ and $\wh \beta_{\rm T}$. More precisely, for any $\lambda \in [0, 1]$ a new predictor is synthesized by setting
\[
    \wh \beta_\lambda = \lambda \beta_{\rm S} + (1-\lambda) \beta_{\rm T}. 
\]
The parameter $\lambda$ represents our \textit{belief} in the explanatory power of the source domain data: if $\lambda = 0$, the source domain has no power to explain the target domain, and we recover $\wh \beta_0 = \beta_{\rm T}$, the empirical target predictor. If $\lambda = 1$, the source domain has an absolute predictive power on the target domain, and it is beneficial to use $\wh \beta_1 = \wh \beta_{\rm S}$ because the sample size $N_{\rm S}$ is large. Discretizing $\lambda$ in the range $[0, 1]$ forms a family of experts $\mc E$.

\textbf{Reweighting Strategy.} Reweighting samples is a common strategy in domain adaptation, transfer learning and adversarial training. \citet{ref:garcke2014importance} synthesize experts, for example, by solving
\[
    \Min{\beta \in \R^d}~  \sum_{i=1}^{N_{\rm S}} w_{h,i} (\beta^\top \wh x_i -  \wh y_i)^2 + \sum_{j=1}^{N_{\rm T}} (\beta^\top \wh x_j -  \wh y_j)^2 +  \eta \| \beta\|_2^2
\]
for some non-negative weights $w_{h,i}$ determined via a Gaussian kernel with bandwidth $h > 0$ of the form
\[
    w_{h,i} = \sum_{l=1}^{N_{\rm S}} \alpha_l \exp \left( - \frac{\| \wh x_i - \wh x_l\|_2^2 + (\wh y_i - \wh y_l)^2}{h^2} \right)
\]
for $i = 1, \ldots, N_{\rm S}$. Here, the parameter vector $\alpha \in \R_+^{N_{\rm S}}$ solves the exponential cone optimization problem
\be \notag 
    \begin{array}{cl}
    \max & \!\! \ds \sum_{j=1}^{N_{\rm T}} \log \Big( \sum_{l=1}^{N_{\rm S}} \alpha_l \exp \Big(\!-\!\frac{\| \wh x_j - \wh x_l\|_2^2 + (\wh y_j - \wh y_l)^2}{h^2} \Big) \Big) \\
    \st & \!\!\ds \sum_{i = 1}^{N_{\rm S}} \sum_{l=1}^{N_{\rm S}} \alpha_l \exp \left( - \frac{\| \wh x_i - \wh x_l\|_2^2 + (\wh y_i - \wh y_l)^2}{h^2}  \right)\!=\!N_{\rm S}. 
        \end{array} 
\ee
The predictor $\beta_h$, parametrized by the kernel weight $h$, that solves the reweighted ridge regression problem has the form
\[
 \Big( \sum_{j=1}^{N_{\rm T}} \wh x_j \wh x_j^\top + \sum_{i=1}^{N_{\rm S}} w_i \wh x_i \wh x_i^\top + \eta I_d \Big)^{-1} \Big( \sum_{j=1}^{N_{\rm T}} \wh x_j \wh y_j  + \sum_{i=1}^{N_{\rm S}} w_i \wh x_i \wh y_i\Big).
\]
Discretizing the bandwidth $h$ forms a family of experts $\mc E$. 

\textbf{Bernstein Online Aggregation (BOA). } We now give a brief overview on the BOA algorithm, which is a recursive expert aggregation procedure for sequential prediction~\citep{ref:cesa2006prediction}. For a given set of experts~$\mc E = \{\beta_{1}, \ldots, \beta_{|\mc E|}\}$ and an incumbent weight $\pi_{k, j-1}$ for expert~$k$ at time~$j-1$, this algorithm aggregates the individual expert's predictions linearly based on the arrival of the input data $(x_j, y_j)$ as $\sum_{k=1}^{|\mc E|} \pi_{k,j} \beta_k^\top x_j$. The weights of the experts are updated using the exponential rule
      \[\pi_{k ,j} = \frac{\exp(-\upsilon(1+\upsilon L_{k,j}) L_{k ,j})\pi_{k, j-1}}{\sum_{k=1}^{|\mc E|} \exp(-\upsilon(1+\upsilon L_{kj}) L_{k ,j})\pi_{k, j-1}},\]
where $\upsilon>0$ is the learning rate and $L_{k,j} \!=\! (\beta_k^\top x_j - y_j)^2 \!-\! \sum_{k=1}^{|\mc E|} (\beta_k^\top x_j -\ y_j)^2 \pi_{k, j-1}$. This algorithm is initialized with weights $\pi_{k,0} \geq 0$ satisfying $\sum_{k=1}^{|\mc E|}\pi_{k, 0} = 1$. The cumulative loss for the stream of test data $(x_j, y_j)_{j=1}^J$ is
\be \label{eq:cumulative-loss}
    \sum_{j=1}^J \left( \sum_{k=1}^{|\mc E|} \pi_{k,j} \beta_k^\top x_j - y_j \right)^2.
\ee
For the square loss, the BOA procedure is optimal for the model selection aggregation problem, that is, the excess risk of its batch version achieves the fast rate of convergence $\log(|\mc E|)/J$ in deviation; see~\citet{ref:wintenberger2017optimal}.

\section{Predictor Generation via Distributionally Robust Linear Regression} 
\label{sect:DRLR}

We now specify our framework to generate the set of competitive experts $\mc E$ for future prediction. Our construction is based on the premises that the source domain carries the explanatory power on the target domain to a certain extent and that the scarce target data can provide directional guidance to pull information from the source data. Moreover, we also leverage ideas from distributionally robust optimization and adversarial training, which have been shown to significantly improve the out-of-sample predictive performance~\cite{ref:duchi2018learning,ref:esfahani2018data, ref:blanchet2019robust, ref:gao2020finite, ref:lam2019recovering}. 

With this in mind, our expert generation scheme blends two elements: a distributional probing strategy and a robust estimation procedure. The distributional probing strategy frames the distribution set $\mbb B$, and then each expert is constructed by solving a distributionally robust least squares estimation problem of the form
\be \label{eq:dro}
	\Inf{\beta \in \R^d} \Sup{\QQ \in \mbb B} \EE_{\QQ}[ (\beta^\top X - Y)^2],
\ee
where $\QQ$ is a joint distribution over $(X, Y)$. Generating a collection of distribution sets $\mbb B$ in a systematic manner and solving~\eqref{eq:dro} for each such set will form a family of experts~$\mc E$. 

In a purely data-driven setting with no additional information, it is attractive to probe into the distributional regions \textit{in between} the empirical source distribution ${\Pnom_{\rm S} = N_{\rm S}^{-1} \sum_{i=1}^{N_{\rm S}} \delta_{(\wh x_i, \wh y_i)}}$ and the empirical target distribution $\Pnom_{\rm T} = N_{\rm T}^{-1} \sum_{j=1}^{N_{\rm T}} \delta_{(\wh x_j, \wh y_j)}$. Because probability distributions reside in infinite-dimensional spaces, framing $\mbb B$ in between $\Pnom_{\rm S}$ and $\Pnom_{\rm T}$ is a non-trivial task. Fortunately, because the expected square loss only depends on the first two moments of the joint distribution of $(X, Y)$, it suffices to prescribe $\mbb B$ using a finite parametrization of distributional moments. To this end, let $p = d+1$ represent the dimension of the joint vector $(X, Y)$. For a given set $\mbb U$ on the space of mean vectors and covariance matrices $\R^{p} \times \PSD^{p}$, we consider $\mbb B$ as the lifted distribution set that contains all distributions whose moments belong to $\mbb U$, that is,
\[
    \mbb B =
\left\{
	\QQ \in \mc M(\R^{p}): \QQ \sim (\m, \cov),~ (\m, \cov) \in \mbb U
\right\},
\]
where $\mc M(\R^p)$ denotes the set of all distributions on $\R^p$, and the notation $\QQ \sim (\m, \cov)$ expresses that $\QQ$ has mean $\m$ and covariance matrix $\cov$. It is convenient to construct the moment information set $\mbb U$ using a divergence on $\R^p \times \PSD^p$.

\begin{definition}[Divergence] \label{def:divergence}
    A divergence $\psi$ on $\R^p \times \PSD^p$ satisfies the following properties:
    \begin{itemize}[leftmargin = 3mm, itemsep=0.5mm]
\vspace{-.3cm}
        \item non-negativity:~for any $(\m, \cov)$, $(\msa, \covsa) \in \R^p \times \PSD^p$, we have $\psi((\m, \cov) \parallel (\msa, \covsa) ) \ge 0$,
        \item indiscernability:~$\psi((\m, \cov)\!\parallel\!(\msa, \covsa))\!=\!0$ implies~$(\m, \cov)\!=\!(\msa, \covsa)$.
    \end{itemize}
\end{definition}

In this paper, we will explore two divergences in the space of mean vectors and covariance matrices that are motivated by popular measures of dissimilarity between distributions.
The divergence $\mathds D$ is motivated by the Kullback-Leibler (KL) divergence.
\begin{definition}[Kullback-Leibler-type divergence] \label{def:KL}
    The divergence $\mathds D$ from tuple $(\m, \cov) \in \R^p \times \PD^p$ to tuple $(\msa, \covsa) \in \R^p \times \PD^p$ amounts to 
    \begin{align*}
        &{\mathds D} \big( (\m, \cov) \parallel (\msa, \covsa) \big) \Let \\
        &(\msa - \m)^\top\covsa^{-1} (\msa - \m)\! + \! \Tr{\cov \covsa^{-1}} - \log\det (\cov \covsa^{-1}) \!-\! p.
    \end{align*}
\end{definition}

In fact $\mathds D$ is equivalent to the KL divergence between two non-degenerate Gaussian distributions $\mc N(\m, \cov)$ and $\mc N(\msa, \covsa)$ (up to a factor of~2). As a consequence, $\mathds D$ is non-negative, and it collapses to 0 if and only if $\cov = \covsa$ and $\m = \msa$. We can also show that $\mathds D$ is affine-invariant. However, we emphasize that $\mathds D$ is not symmetric and $\mathds D\big((\m, \cov) \parallel (\msa, \covsa) \big) \neq \mathds D\big((\msa, \covsa) \parallel (\m, \cov) \big)$ in general.

We also study the divergence $\mathds W$ which is motivated by the Wasserstein distance.

\begin{definition}[Wasserstein-type divergence] \label{def:Wasserstein}
	The divergence $\mathds W$ between two tuples $(\m, \cov) \in \R^p \times \PSD^p$ and $(\msa, \covsa) \in \R^p \times \PSD^p$ amounts to 
	\begin{equation*}
	\mathds W \big( (\m, \cov) \! \parallel \! (\msa, \covsa) \big) \!\Let\!  \| \m - \msa \|_2^2 + \Tr{\cov + \covsa - 2 \big( \covsa^{\half} \cov \covsa^{\half} \big)^{\half} } \!.
	\end{equation*}
\end{definition}
The divergence $\mathds W$ coincides with the \textit{squared} type-$2$ Wasserstein distance between two Gaussian distributions $\mc N(\m, \cov)$ and $\mc N(\msa, \covsa)$~\cite{ref:givens1984class}. One can readily show that $\mathds W$ is non-negative, and it vanishes if and only if $(\m, \cov)\!\!=\!\!(\msa, \covsa)$. Thus, $\mathds W$ is a symmetric divergence.

In Sections~\ref{sec:IR} and \ref{sec:SI} we examine in detail two strategies to frame $\mbb U$ and its corresponding distribution set $\mbb B$ in a principled manner, and we devise optimization techniques to solve the resulting robust estimation problems.

\section{``Interpolate, then Robustify'' Strategy}
\label{sec:IR}

``Interpolate, then Robustify'' (IR) is an intuitive strategy to systematically probe into distributional regions between $\Pnom_{\rm S}$ and $\Pnom_{\rm T}$. Let $(\msa_{\rm S}, \covsa_{\rm S})$ be the empirical mean vector and covariance matrix of $\Pnom_{\rm S}$, that is,
\[
    \msa_{\rm S} = \frac{1}{N_{\rm S}} \sum_{i=1}^{N_{\rm S}} \begin{pmatrix}
    \wh x_i \\ \wh y_i
    \end{pmatrix},~\covsa_{\rm S} =\frac{1}{N_{\rm S}} \sum_{i=1}^{N_{\rm S}} \begin{pmatrix}
    \wh x_i \\ \wh y_i
    \end{pmatrix}\begin{pmatrix}
    \wh x_i \\ \wh y_i
    \end{pmatrix}^\top - \msa_{\rm S} \msa_{\rm S}^\top,
\]
and let $(\msa_{\rm T}, \covsa_{\rm T})$ be defined analogously for $\Pnom_{\rm T}$. The IR strategy applies repeatedly the following two steps to generate distribution sets. First, interpolate between $(\msa_{\rm S}, \covsa_{\rm S})$ and $(\msa_{\rm T}, \covsa_{\rm T})$ to obtain a new pair $(\msa_\lambda, \covsa_\lambda)$ parametrized by $\lambda \in [0, 1]$. Second, construct a moment set $\mbb U_{\lambda, \rho}$ as a ball of radius $\rho$ circumscribing the pair $(\msa_\lambda, \covsa_\lambda)$, then lift the moment set $\mbb U_{\lambda, \rho}$ to the corresponding distribution set $\mbb B_{\lambda, \rho}$. More specifically, $(\msa_\lambda, \covsa_\lambda)$ is the $\psi$-barycenter between $(\msa_{\rm S}, \covsa_{\rm S})$ and $(\msa_{\rm T}, \covsa_{\rm T})$, which is obtained by solving
\begin{equation}
\begin{array}{c@{\,}lll}
        \min\limits_{\mu \in \R^p, \Sigma \in \mbb{S}^p_+}& \lambda \mathds \psi((\m, \cov)\! \parallel\! (\msa_{\rm S}, \covsa_{\rm S})) + \\
        &\hspace{1cm} (1\! -\! \lambda) \mathds \psi( (\m, \cov) \! \parallel \! (\msa_{\rm T}, \covsa_{\rm T})).
\end{array}
\label{eq:mean_cov_interpolation}
\end{equation}
\begin{figure}[ht!]
    \centering
    \vspace{-2mm}
    \includegraphics[width=\columnwidth]{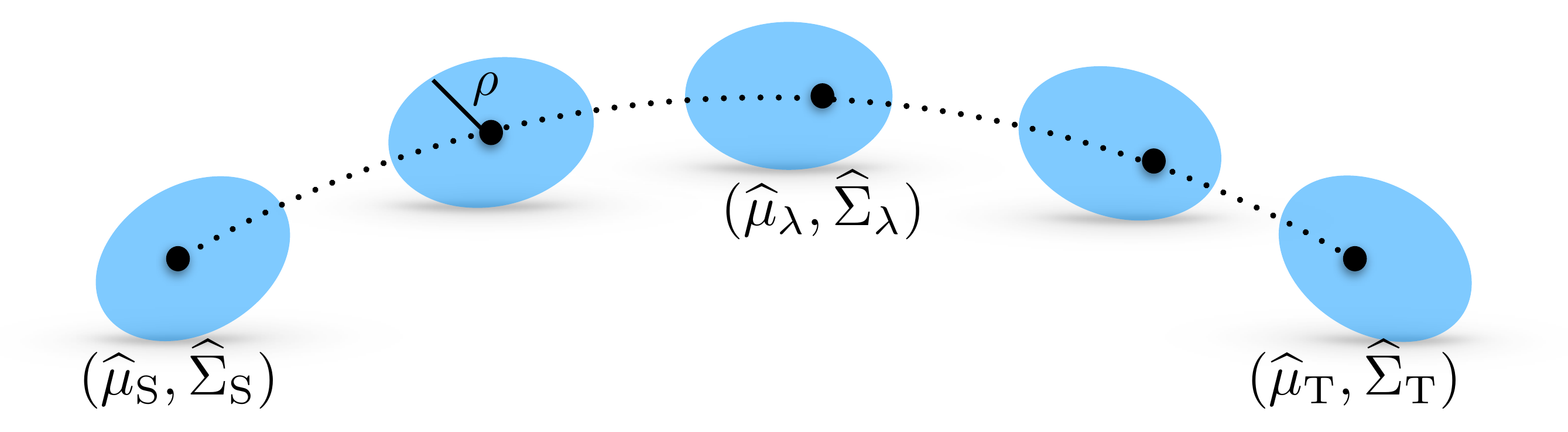} 
    \vspace{-8mm}
    \caption{The dashed curve shows the barycenter interpolations parametrized by $\lambda\in[0,1]$. Ellipses represent $\mbb U_{\lambda, \rho}$ at different $\lambda$.}
    \label{fig:ir_strategy}
\end{figure}

Then, we employ the divergence $\psi$ to construct an uncertainty set $\mbb U_{\lambda, \rho}$ in the mean-covariance matrix space as 
\[
    \mbb U_{\lambda, \rho} \Let \left\{ (\m, \cov)\in \R^p\times \PSD^p: \psi( (\m, \cov)\! \parallel\! (\msa_\lambda, \covsa_\lambda)) \le \rho \right\}.
\]
The outlined procedure is illustrated in Figure~\ref{fig:ir_strategy}. An expert is now obtained by solving the distributionally robust least squares problem~\eqref{eq:dro} with respect to the distribution set
\[\mbb B_{\lambda,\rho} = \{\QQ \in \mc M(\R^p) : \QQ \sim (\mu, \Sigma), (\mu, \Sigma) \in \mbb U_{\lambda, \rho}\}.\]
Notice that in this strategy the parameter $\lambda \in [0, 1]$ characterizes the explanatory power of the source domain to the target domain: if $\lambda = 0$, then $(\msa_\lambda, \covsa_\lambda) = (\msa_{\rm T}, \covsa_{\rm T})$, and if $\lambda = 1$, then $(\msa_\lambda, \covsa_\lambda) = (\msa_{\rm S}, \covsa_{\rm S})$. Thus, as $\lambda$ decreases, $(\msa_\lambda, \covsa_\lambda)$ is moving farther away from the source information $(\msa_{\rm S}, \covsa_{\rm S})$, and $(\msa_\lambda, \covsa_\lambda)$ is pulled towards the target information $(\msa_{\rm T}, \covsa_{\rm T})$. 

The choice of the divergence $\psi$ influences both the barycenter problem~\eqref{eq:mean_cov_interpolation} and the formation of the set $\mbb U_{\lambda, \rho}$. Next, we study the special case of the IR strategy with the KL-type divergence and the Wasserstein-type divergence.

    
    

\subsection{Kullback-Leibler-type Divergence}

The KL-type divergence $\mathds D$ in Definition~\ref{def:KL} is not symmetric. Hence, it is worthwhile to note that the barycenter problem~\eqref{eq:mean_cov_interpolation} optimizes over $(\m, \cov)$ being placed in the first argument of $\mathds D$, and that the set $\mbb U_{\lambda, \rho}$ is also defined with the pair $(\m, \cov)$ being placed in the first argument. Under the divergence $\mathds D$, the barycenter $(\msa_\lambda, \covsa_\lambda)$ admits a closed form expression. This fact is well-known in the field of KL fusion of Gaussian distributions~\cite{ref:battistelli2013consensus}. 
\begin{proposition}[KL barycenter]
\label{prop:kl_interpolation}
Suppose that  $\psi $ is the KL-type divergence.
If $\wh \cov_{\rm S}, \wh \cov_{\rm T} \succ 0$, then $(\msa_\lambda, \covsa_\lambda)$ is the minimizer of the barycenter problem~\eqref{eq:mean_cov_interpolation} with
\begin{align*}
            \wh \Sigma_\lambda &= (\lambda \covsa_{\rm S}^{-1} + (1-\lambda) \covsa_{\rm T}^{-1})^{-1} \succ 0,\\
            \wh \mu_\lambda &= \covsa_\lambda \big( \lambda \covsa_{\rm S}^{-1} \msa_{\rm S} + (1-\lambda) \covsa_{\rm T}^{-1} \msa_{\rm T} \big).
\end{align*}
\end{proposition}


For a given $\lambda \in [0, 1]$ and $\rho \ge 0$, the corresponding IR-KL expert is obtained by solving
    \be  
    \label{opt:KL_expert}
        \Min{\beta \in \R^d}~\left\{f_{\lambda, \rho}(\beta) \Let
        \Sup{\QQ \in \mbb B_{\lambda,\rho}} \EE_\QQ[(\beta^\top X - Y)^2] \right\}.
    \ee
Problem~\eqref{opt:KL_expert} can be efficiently solved using a gradient-descent algorithm. To do this, the next proposition establishes the relevant properties of $f_{\lambda, \rho}$.

\begin{proposition}[Properties of $f_{\lambda, \rho}$] \label{prop:grad_f_D}

    The function $f_{\lambda, \rho}$ is convex and continuously differentiable with
    \begin{align*}\label{eq:grad_f_D}
        \nabla f_{\lambda, \rho} (\beta) \!=\! \frac{2\dualvar\opt \left(  \omega_2  \covsa_\lambda w \!+\! (\dualvar\opt \! - \!\omega_1 ) (\covsa_\lambda \! +\! \msa_\lambda \msa_\lambda^\top) w \right)_{1:d}}{(\dualvar\opt  - \omega_1 )^2}  ,
    \end{align*}
    where 
    $w = [\beta^\top, -1]^\top$, $\omega_1 = w^\top \covsa_\lambda w $, $\omega_2  = (w^\top \msa)^2$ and $\dualvar\opt\in ( \omega_1 , \omega_1  \big(1 + 2\rho + \sqrt{1 + 4 \rho\, \omega_2 } \big)/(2\rho)]$ is the unique solution of the equation
    \begin{equation*}
        \rho = (\dualvar - \omega_1 )^{-2} \omega_1 \omega_2  + (\dualvar - \omega_1 )^{-1} \omega_1  + \log( 1 - \dualvar^{-1} \omega_1 ).
    \end{equation*}
    Furthermore, $f_{\lambda, \rho}$ is locally smooth at any $\beta \in \R^d$, \ie, there exist constants $C_\beta, \epsilon_\beta> 0$ such that for any $\beta'\in\R^d$ with $\| \beta' - \beta \|_2 \le \epsilon_\beta$,
we have $\| \nabla f_{\lambda, \rho} (\beta') - \nabla f_{\lambda, \rho} (\beta) \|_2 \le C_\beta \norm{\beta' - \beta}_2$.
\end{proposition}

Thanks to Proposition~\ref{prop:grad_f_D}, we can apply the adaptive gradient method to solve problem~\eqref{opt:KL_expert} to global optimality, and the algorithm enjoys a sublinear rate $| f_{\lambda, \rho}(\bar{\beta}^k) - f_{\lambda, \rho}(\beta_{\lambda, \rho}\opt ) | \le O(k^{-1})$, where $\bar{\beta}^k$ is a certain average of the iterates, and $\beta_{\lambda, \rho}\opt $ is an optimal solution of~\eqref{opt:KL_expert}. The algorithm and its guarantees are detailed in~\citet{malitsky2019adaptive}.

\begin{figure*}[ht!]
    \centering
    \includegraphics[width=\textwidth]{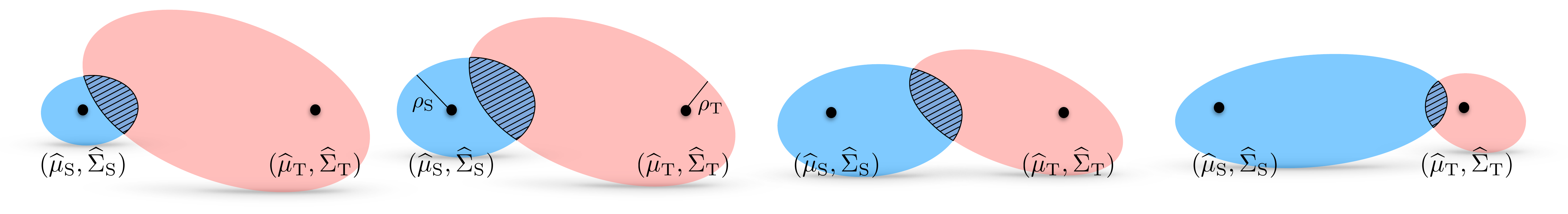}
    \vspace{-8mm}
    \caption{Varying $(\rho_{\rm S}, \rho_{\rm T})$ frames different moment sets $\mbb U_{\rho_{\rm S}, \rho_{\rm T}}$ (hatched regions). The radius $\rho_{\rm S}$ increases from left to right.}
    \label{fig:SI}
    \vspace{-3mm}
\end{figure*}

\subsection{Wasserstein-type Divergence}
Under the divergence $\mathds W$ in Definition~\ref{def:Wasserstein}, problem~\eqref{eq:mean_cov_interpolation} resembles the Wasserstein barycenter in the space of Gaussian distributions. The result from \citet[\S6.2]{agueh2011barycenters} implies that the barycenter~$(\msa_\lambda, \covsa_\lambda)$ admits a closed form expression following the McCann's interpolant~\citep[Example~1.7]{ref:mccann1997convexity}. 



\begin{proposition}[Wasserstein interpolation]
\label{prop:wass_interpolation}
Suppose that $\psi$ is the Wasserstein-type divergence. If $\covsa_{\rm S} \succ 0$, then $(\msa_\lambda, \covsa_\lambda)$ is the minimizer of problem~\eqref{eq:mean_cov_interpolation} with
\begin{align*}
    \wh \m_\lambda &= \lambda \msa_{\rm S} + (1-\lambda) \msa_{\rm T}, \\
    \wh \cov_\lambda &= (\lambda I_p + (1-\lambda) L) \covsa_{\rm S} (\lambda I_p + (1-\lambda) L) ,
\end{align*}
where $L = \covsa_{\rm T}^\half (\covsa_{\rm T}^\half \covsa_{\rm S} \covsa_{\rm T}^\half)^{-\half} \covsa_{\rm T}^\half$.
\end{proposition}
For a given $\lambda \in [0, 1]$ and $\rho \ge 0$, we obtain the corresponding IR-Wasserstein expert by solving a conic program using off-the-shelf solvers such as~\citet{mosek}.
\begin{proposition}[IR-Wasserstein expert] \label{prop:IR_W}
Suppose that $\psi$ is the Wasserstein-type divergence. Problem~\eqref{eq:dro} with $\mbb B \equiv \mbb B_{\lambda, \rho}$ is equivalent to the second order cone program
\[
    \min\limits_{\beta \in \R^d} ~ \left\|(\covsa_\lambda + \msa_\lambda \msa_\lambda^\top)^{\half} \begin{bmatrix}
        \beta \\
        -1
    \end{bmatrix} \right\|_2 + \sqrt{\rho} \left\| \begin{bmatrix}
        \beta \\
        -1
    \end{bmatrix} \right\|_2.
\]
\end{proposition}

\section{``Surround, then Intersect'' Strategy}
\label{sec:SI}

``Surround, then Intersect'' (SI) probes naturally into the distributional space by intersecting two balls centered at the empirical moments. More specifically, this strategy circumscribes $(\msa_{\rm S}, \covsa_{\rm S})$ (respectively, $(\msa_{\rm T}, \covsa_{\rm T})$) with a ball of radius $\rho_{\rm S}$ (respectively, $\rho_{\rm T}$) using the $\psi$-divergence. Consequentially, the moment information set $\mbb U_{\rho_{\rm S}, \rho_{\rm T}}$ in the mean vector-covariance matrix space is defined as
\begin{equation*}
\mbb U_{\rho_{\rm S}, \rho_{\rm T}}
\!\Let \!\left\{
\begin{array}{l}
(\m, \cov)\in \R^p\times \PSD^p \text{ such that: } \\
\psi((\mu, \Sigma) \! \parallel \! (\wh \mu_{\rm S}, \wh \Sigma_{\rm S})) \leq \rho_{\rm S} \!\\
\psi((\mu, \Sigma) \! \parallel \! (\wh \mu_{\rm T}, \wh \Sigma_{\rm T})) \leq \rho_{\rm T}\!\\
\cov + \m \m^\top \succeq \eps I_p
\end{array}  \! \!\right\},
\end{equation*}
where the small constant $\eps > 0$ improves numerical stability. This construction is graphically illustrated in Figure~\ref{fig:SI}.
An expert is now obtained by solving the distributionally robust least squares problem~\eqref{eq:dro} subject to the distributional set
\[
    \mbb B_{\rho_{\rm S}, \rho_{\rm T}} =  \left\{ 
        \QQ \in \mc M(\R^{p}): \QQ \sim (\m, \cov),~ (\m, \cov) \in \mbb U_{\rho_{\rm S}, \rho_{\rm T}}
	\right\}.
\]
Note that $\mbb B_{\rho_{\rm S}, \rho_{\rm T}}$ is well-defined only if the radii $(\rho_{\rm S}, \rho_{\rm T})$ are sufficiently large so that the intersection of the two balls becomes non-empty. A sensible approach to set these parameters is to fix~$\rho_{\rm S}$ and to find a sufficiently large~$\rho_{\rm T}$ so that~$\mbb U_{\rho_{\rm S}, \rho_{\rm T}}$ is non-empty. In this way, the SI strategy characterizes the explanatory power of the source domain to the target domain by the radius $\rho_{\rm S}$: if $\rho_{\rm S} = 0$ then $\mbb U_{\rho_{\rm S}, \rho_{\rm T}}$ becomes a singleton $\{(\msa_{\rm S}, \covsa_{\rm S})\}$, representing the \textit{belief} that the source domain possess absolute explanatory power onto the target domain. As $\rho_{\rm S}$ increases, $\mbb U_{\rho_{\rm S}, \rho_{\rm T}}$ is gradually pulled towards the empirical target moments $(\msa_{\rm T}, \covsa_{\rm T})$. Next, we study the special case of the SI strategy with the KL-type divergence and the Wasserstein-type divergence.

\subsection{Kullback-Leibler-type Divergence}
Recall that~$\mathds D$ is asymmetric and $(\m, \cov)$ is the first argument of $\mathds D$ in the definition of~$\mbb U_{\rho_{\rm S}, \rho_{\rm T}}$. We first study conditions on $\rho_{\rm T}$ under which the ambiguity set~$\mbb B_{\rho_{\rm S}, \rho_{\rm T}}$ is non-empty.
\begin{proposition}[Minimum radius]
\label{prop:minimum_radius_kl}
Suppose that $\psi$ is the KL-type divergence. For any $\rho_S > 0$ the sets $\mbb U_{\rho_{\rm S}, \rho_{\rm T}}$ and $\mbb B_{\rho_{\rm S}, \rho_{\rm T}}$ are non-empty if  $\rho_{\rm T} \ge \mathds{D}((\msa_{\gamma\opt}, \covsa_{\gamma\opt}) \parallel (\msa_{\rm T}, \covsa_{\rm T}))$, where $\gamma\opt$ is a maximizer of
\[
\begin{array}{cl}
\sup & \!\!\mathds{D}((\msa_{\gamma}, \covsa_{\gamma})\!\! \parallel\!\! (\msa_{\rm S}, \covsa_{\rm S}))\!+\! \mathds{D}((\msa_{\gamma}, \covsa_{\gamma}) \!\! \parallel\! \!(\msa_{\rm T}, \covsa_{\rm T}))\!-\!\gamma \rho_{\rm S} \\
\st & \gamma \in \R_+, \covsa_{\gamma} = (1+\gamma) (\gamma \covsa_{\rm S}^{-1} + \covsa_{\rm T}^{-1})^{-1} \in \PSD^p,  \\
& \msa_{\gamma} = \covsa_{\gamma} (\gamma \covsa_{\rm S}^{-1} \msa_{\rm S} + \covsa_{\rm T}^{-1} \msa_{\rm T}) / (1+\gamma) \in \R^p
\end{array}
\]
\end{proposition}
The above optimization problem is effectively  one-dimensional and can therefore be solved by bisection on~$\gamma$. The next theorem asserts that the SI-KL experts are formed by solving a semidefinite program.
\begin{theorem}[SI-KL Expert] \label{thm:ls-kl}
Suppose that $\psi$ is the KL-type divergence and $\mbb B \equiv \mbb B_{\rho_{\rm S}, \rho_{\rm T}}$ is non-empty. Then $\beta\opt = (M_{XX}\opt)^{-1} M_{XY}\opt$ solves problem~\eqref{eq:dro}, where $(M_{XX}\opt, M_{XY}\opt)$ is a solution 
of the convex semidefinite program
    \be \notag
\begin{array}{cl}
    \sup & \tau \\
    \st & M_{XX} \in \R^{d \times d},~ M_{XY} \in \R^{d \times 1},~M_{YY} \in \R\\
    & \tau \in \R_+,~\m \in \R^{p},~M\in \PD^{p}, t \in \R_+\\
    &\msa_k^\top \covsa_k^{-1} \msa_k - 2 \msa_k^\top \covsa_k^{-1} \m+\Tr{M\covsa_k^{-1}} -\\
    &\log\det (M\covsa_k^{-1}) \!-\!
    \log(1\!-\! t) - p \!\leq\! \rho_k \,~ \forall k \! \in \! \{\rm{S}, \rm{T}\} \\[1ex]
	&\begin{bmatrix} M & \m \\ \m^\top & t \end{bmatrix} \succeq 0,~\begin{bmatrix}
	    M_{XX} & M_{XY} \\ M_{XY}^\top & M_{YY}-\tau
	\end{bmatrix} \succeq 0\\
	&M = \begin{bmatrix}
	    M_{XX} & M_{XY} \\ M_{XY}^\top & M_{YY}
	\end{bmatrix} \succeq \eps I_p.
\end{array}
\ee
\end{theorem}

\subsection{Wasserstein-type Divergence}
The space $\R^p\times \PSD^p$ can be endowed with a distance inherited from the Wasserstein distance between Gaussian distribution. For any $\rho_{\rm S} > 0$, the minimum radius for~$\rho_{\rm T}$ that makes~$\mbb B_{\rho_{\rm S},\rho_{ \rm T}}$ non-empty is known in closed form.
\begin{proposition}[Minimum radius]
\label{prop:minimum_radius}
Suppose that $\psi$ is the Wasserstein-type divergence. For any $\rho_{\rm S} >0$ the sets $\mbb U_{\rho_{\rm S}, \rho_{\rm T}}$ and $\mbb B_{\rho_{\rm S}, \rho_{\rm T}}$ are non-empty if
\[\rho_{\rm T}\geq \left(\sqrt{\mathds  W((\msa_{\rm S}, \covsa_{\rm S}) \parallel (\msa_{\rm T}, \covsa_{\rm T}))} - \sqrt{\rho_{\rm S}}\right)^{2}.\]
\end{proposition}
The next theorem asserts that the SI-Wasserstein experts are constructed by solving a semidefinite program.
\begin{theorem}[SI-Wasserstein expert] \label{thm:ls-w}
Suppose that $\psi$ is the Wasserstein-type divergence and $\mbb B \equiv \mbb B_{\rho_{\rm S}, \rho_{\rm T}}$ is non-empty. Then $\beta^\star \!=\! (M\opt_{XX})^{-1} M_{XY}\opt$ solves problem~\eqref{eq:dro}, where $(M\opt_{XX}, M\opt_{XY})$ is a solution 
of the linear semidefinite program
\begin{equation} \notag
\!\!\!\begin{array}{cl}
     \sup &\!\tau  \\
     \st & M_{XX} \in \R^{d \times d}, M_{XY}\in \R^{d\times 1}, M_{YY} \in \R\\
     &\tau \in \R_+, \mu \in \R^p, M, H \in \PSD^{p}, C_{\rm S}, C_{\rm T} \in \R^{p \times p} \\
     &\hspace{-.2cm}\left.\begin{array}{l}
          \| \wh \mu_k\|_2^2 \!-\! 2 \msa_k^\top \m \! +\! \Tr{\! M \!+\!\wh \Sigma_k  \! -\! 2 C_k} \! \leq\! \rho_k   \\
          \begin{bmatrix}
         H &C_k \\
         C_k^\top &\wh \Sigma_k
     \end{bmatrix} \succeq 0
     \end{array}\!\!\!\!\right\}\!k \! \in \!\{\rm S, \rm T\}\\ [1.5ex]
     &\begin{bmatrix}
         M-H &\mu \\
         \mu^\top &
     \end{bmatrix} \succeq 0\\  [1.5ex]
     &\begin{bmatrix}
         M_{XX} &M_{XY} \\
         M_{XY}^\top &M_{YY}\!\!
        -\!\tau
     \end{bmatrix}\! \!\succeq \! 0,\,M\!\!=\!\!
     \begin{bmatrix}
         M_{XX} &M_{XY}\\
         M_{XY}^\top &M_{YY}\end{bmatrix} \!\!\succeq\!\eps I_{\!p}.
\end{array}
\end{equation}

\end{theorem}

\section{Numerical Experiments}
\label{sec:numerical}

\begin{table*}[h!]
    \centering
    \scriptsize
    \begin{tabular}{||c|c|c|c|c|c|c|c|c|c|c|c|c|c||}
    \hline
    \hline
    Data Set&Time &IR-KL &IR-WASS &SI-KL &SI-WASS & CC-L &CC-TL &CC-SL & CC-TE &CC-SE &RWS &LSE-T &LSE-T$\&$S \\
    \hline
    \hline
    \multirow{4}{*}{Uber$\&$Lyft} 
    &5 &17.65	&\textbf{1.00}	&199.28 &1.01 &34.04 &98.43 &12.03 &155.71 &1.74 &1.45 &119.65 &11.08\\
    &10 &13.67 &\textbf{1.00} &111.52 &1.01 &30.85 &99.22 & 11.40 &161.72 &1.58 &1.34 &137.15 &6.32 \\
    &50 &13.39 &\textbf{1.00}	&60.29 &1.01 &25.87 &	85.06 &9.72 &147.45 &1.42 &1.16 &57.85 &2.12\\
    &100 &15.24 &\textbf{1.00}	&59.06 &1.01 &26.01 &	85.77 &9.91 &148.49 &1.41 &1.12 &31.25 &1.57 \\
    \hline
    \hline
    \multirow{4}{*}{$\substack{\text{US}\\ \text{Births (2018)}}$}
    &5 &79.83 &1.02 &44.71 &\textbf{1.00}	&64.99 &	257.60 &25.13 &432.09 &2.07 &4.50 &727.88 &39.17 \\
    &10 & 115.47 &1.02 &39.35 &\textbf{1.00} &45.59 &195.14 &18.33 &339.11 &1.60 &3.29 &524.39 &19.28\\
    &50 &107.40 &1.01 &40.04 &\textbf{1.00}	 &42.74 &	192.46 &13.12 &361.51 &1.31 &2.00 &191.27 &5.20\\
    &100 &117.03 &1.01 &53.13 &\textbf{1.00} &45.35 &	208.65 &12.94 &397.33 &1.22 &1.75 &104.75 &3.19\\
    \hline
    \hline
    \multirow{4}{*}{$\substack{\text{Life}\\ \text{Expectancy}}$}  
    &5 &33.18 &\textbf{1.00} &6.24 &1.03 &17.24 &77.06 &7.38 &125.71 &1.46 &1.15 &255.08 &20.72\\
    &10 &25.59 &\textbf{1.00}	&5.45 &1.02 &12.49 &60.19 &5.50 &104.00 &1.40 &1.15 &167.15 &10.73\\
    &50 &19.81 &\textbf{1.00}	&8.70 &1.01 &7.57 &44.00 &3.10 &84.98 &1.38 &1.10 &39.83 &3.15\\
    &100 &19.02 &\textbf{1.00}	&8.25 &1.005 &6.82 &41.40 &2.68 &83.60 &1.38 &1.08 &20.42 &2.10  \\
    \hline
    \hline
    \multirow{4}{*}{$\substack{\text{House}\\ \text{Prices in KC}}$}  
    &5 &1.58 &\textbf{1.00}	&1.21 &1.01 &3.98 &8.87 &2.12 &13.31 &1.29 &1.23 &11.75 &3.70\\
    &10 &1.52 &\textbf{1.00}	&1.20 &1.01 &3.58 &7.77 &2.02 &11.70 &1.27 &1.23 &6.93 &2.25\\
    &50 &1.34 &\textbf{1.00} &1.31 &1.01 &2.79 &6.52 &1.86 &10.37 &1.27 &1.20 &3.91 &1.30\\
    &100 &1.34 &\textbf{1.00}	&1.30 &1.01 &2.65 &6.54 &1.91 &10.74 &1.27 &1.18 &2.72 &1.12\\
    \hline
    \hline
    \multirow{4}{*}{$\substack{\text{California}\\ \text{Housing}}$}  
    &5 &63.33 &1.05 &3.31 &\textbf{1.00} &27.63 &102.82 &9.60 &181.52 &1.35 &1.17 &96.43 &54.34\\
    &10 &68.08 &1.04 &2.42 &\textbf{1.00}	&20.57 &91.86 &6.23 &169.87 &1.19 &1.17 &45.64 &24.76\\
    &50 &70.08 &1.01	&1.97	&\textbf{1.00}	&11.79 &	81.72 &2.49 &170.18 &1.05 &1.13 &10.17 &5.63\\
    &100 & 72.80 &1.003 &1.90 &\textbf{1.00}	&9.71 &79.19 &1.83 &173.96 &1.04 &1.14 &5.81 &3.39\\
    \hline
    \hline
    \end{tabular} 
    \vspace{-.2cm}
    \caption{Normalized cumulative loss values averaged over~100 independent runs. }
    \vspace{-4mm}
    \label{tab:cum_loss_table}
\end{table*}

The second-order cone and semidefinite programs are modelled in MATLAB via YALMIP~\citep{ref:lofberg2004yalmip} and solved with~\citet{mosek}. All experiments are run on an~Intel i7-8700 CPU (3.2 GHz) computer with 16GB RAM. The corresponding codes are available at~\url{https://github.com/RAO-EPFL/DR-DA.git}.

We now aim to assess the performance of experts and demonstrate the effects of robustness.
In all experiments we generate the set $\mc E = \{\beta_1, \ldots, \beta_{|\mc E|}\}$ of experts with $|\mc E| = 10$. 

We consider four family of robust experts generated by:
\begin{itemize}[leftmargin =3mm, itemsep=0.5mm]
\vspace{-.3cm}
    \item IR-KL: with $\rho\!=\!\mathds D((\msa_{\rm T}, \covsa_{\rm T})\! \parallel\! (\msa_{\rm S}, \covsa_{\rm S}) ) / (3|\mc E|)$ and $\lambda$ is spaced from~1 to~0 in exponentially increasing steps.\footnote{We say that~$\lambda$ is spaced from~$a$ to~$b$ in~$K$ exponentially increasing steps if~$\lambda_1 = a$ and~$\lambda_{k+1} = \lambda_{k} - (a  - b) \exp(k)/ \sum_{i=1}^{K-1} \exp(i)$ for all $k\in\{2,\ldots, K-1\}$.}
    
    \item IR-WASS: with~$\rho\!\!=\!\!\mathds W((\msa_{\rm T}, \covsa_{\rm T}) \!\!\parallel\!\! (\msa_{\rm S}, \covsa_{\rm S}) ) / (3|\mc E|)$ and~$\lambda$ is spaced from~1 to~0 in exponentially increasing steps.
    
    \item SI-KL: with~$\rho_{\rm S}$ spaced from~$10^{-3}$ to~$\mathds D((\msa_{\rm T}, \covsa_{\rm T})~\!\!\parallel\!\!~(\msa_{\rm S}, \covsa_{\rm S}) )\!-\!1$ in exponentially increasing steps.
    For a given $\rho_{\rm S}$, ~$\rho_{\rm T} $ is set to the sum of the minimum target radius satisfying the condition of~Proposition~\ref{prop:minimum_radius_kl} and $\rho_{\rm S} / 2$.\footnote{If $d\geq 15$, then the minimum value of~$\rho_{\rm S}$ is set to~5 to improve numerical stability.}
    \item SI-WASS: with $\rho_S$ spaced from~$10^{-4}$ to~$\mathds W((\msa_{\rm T}, \covsa_{\rm T}) \parallel (\msa_{\rm S}, \covsa_{\rm S}) )$ in increasing exponential steps. For a given $\rho_{\rm S}$, $\rho_{\rm T} $ is set to the sum of the minimum radius that satisfies the condition in~Proposition~\ref{prop:minimum_radius} and~$\rho_{\rm S} / 2$.
    \end{itemize}
    \vspace{-.3cm}
    We benchmark against the Convex Combination~(CC) and Reweighting~(RW) experts in Section~\ref{sec:problem} generated by 
    \begin{itemize}[leftmargin =4mm, itemsep=0.5mm]
    \vspace{-.3cm}
        \item  CC-L:~with~$\lambda$ equally spaced in~$[0, 1]$, thus provides uniformly spaced distributional regions in between domains.
        \item CC-TL: with~$\lambda$ equally spaced in~$[0, 0.5]$, thus distributional regions are formed around the target domain.
        \item CC-SL: with~$\lambda$ equally spaced in~$[0.5, 1]$, thus distributional regions are formed around the source domain.
        \item CC-TE: with $\lambda$ spaced from~0 to~1 in exponentially increasing steps, thus the constructed distributional regions are concentrated towards the target domain. 
        \item CC-SE: with~$\lambda$ spaced from~1 to~0 in exponentially increasing steps, thus the constructed distributional regions are concentrated towards the source domain. 
        \item RWS: with~$h$ equally spaced in~$[0.5, 10]$.
    \end{itemize}
    \vspace{-.3cm}
    
    We consider a family of sequential empirical ridge regression estimators generated by training for each~$J$ over
    \begin{itemize}[leftmargin =4mm, itemsep=0.5mm]
    \vspace{-.3cm}
        \item LSE-T, the union of the target dataset~$(\wh x_j, \wh y_j)_{j=1}^{N_{\rm T}}$, and the sequentially arriving target test data~$(x_j, y_j)_{j=1}^{J-1}$,
        \item LSE-T$\&$S, the union of the source data~$(\wh x_i, \wh y_i)_{i=1}^{N_{\rm S}}$, the target data~$(\wh x_j, \wh y_j)_{j=1}^{N_{\rm T}}$ and the sequentially arriving target test data~$(x_j, y_j)_{j=1}^{J-1}$.
    \end{itemize} 
\vspace{-.2cm}
 Note that both LSE-T and LSE-T$\&$S predictors dynamically incorporate the new data to adapt the prediction. Thereby, they have an unfair advantage  in the long run over the other experts that are trained only once at the beginning with $N_{\rm T}$ samples from the test domain. 
 
 The main reason behind using exponential step sizes originates from the asymmetric nature of~$\mathds D$. For simplicity, we also use it for experts with~$\mathds W$.
To ensure fairness in the competition between experts, we vary the parameters of the non-robust experts also in exponential steps.

We compare the performance of our model against the above non-robust benchmarks on~5 Kaggle datasets:\footnote{Descriptions and download links are provided in the appendix.}
    \begin{itemize}[leftmargin = 3mm, itemsep=0.5mm]
    \vspace{-0.3cm}
          \item \textbf{Uber$\&$Lyft} contains~$d\!=\!38$ features of Uber and Lyft cab rides in Boston including the distances, date and time of the hailing, a weather summary for that day. The prediction target is the price of the ride. We divide the dataset based on the company, Uber~(source) and Lyft~(target). 
         
        \item \textbf{US Births~(2018)} has~$d=36$ predictive features of child births in the United States in the year of~2018 including the gender of the infant, mother's weight gain, and mother's per-pregnancy body mass index. The task is to predict the weight of the infants. We divide the dataset based on gender: male (source) and female (target).
        
       \item \textbf{Life Expectancy} contains $d = 19$ predictive features, and the target variable is the life expectancy at birth. The dataset is divided into  
       two subgroups: developing (source) and developed (target) countries.
        
        \item \textbf{House Prices in King Country} contains $d \!=\! 14$ predictive variables, the target variable is the transaction price of the houses. We split the dataset into two domains: houses built in~$[1950, 2000)$~(source) and~$[2000, 2010]$~(target).
         \item \textbf{California Housing Prices} has~$d = 9$ predictive features, the target variable is the price of houses. We divide this dataset into houses with less than an hour drive to the ocean shore~(source) and houses in inland~(target).
    \end{itemize}
    \vspace{-.3cm}
    We use all samples from the source domain for training, and we form the target training set by drawing $N_{\rm T}\!=\!d$ samples from the target dataset.
    Later, we randomly sample~$J\!=\!1000$ data points from the remaining target samples to form the sequentially arriving target test samples. 
    Note that the performance of the experts is sensitive to the data, and thus we replicate this procedure~100 times.
    We set the regularization parameter of the ridge regression problem to $\eta = 10^{-6}$ and the learning rate of the BOA algorithm to $\upsilon = 0.5$.
    We measure the performance of the experts by the cumulative loss~\eqref{eq:cumulative-loss} calculated for every~$J$.
    
    Table~\ref{tab:cum_loss_table} shows the average cumulative loss of each aggregated expert obtained by the BOA algorithm for all datasets and for~$J\! =\! \{5, 10, 50, 100\}$ across~100 independent runs.
    In each row, the minimum loss is normalized to 1, and the remaining entries are presented by the multiplicative factor of the minimum value.
    This result suggests that the IR-WASS and SI-WASS experts perform favorably over the competitors in that their cumulative loss at each time step is substantially lower than that of most other competitors.
    \begin{figure}[h!]
    \centering
    \includegraphics[width=.9\columnwidth]{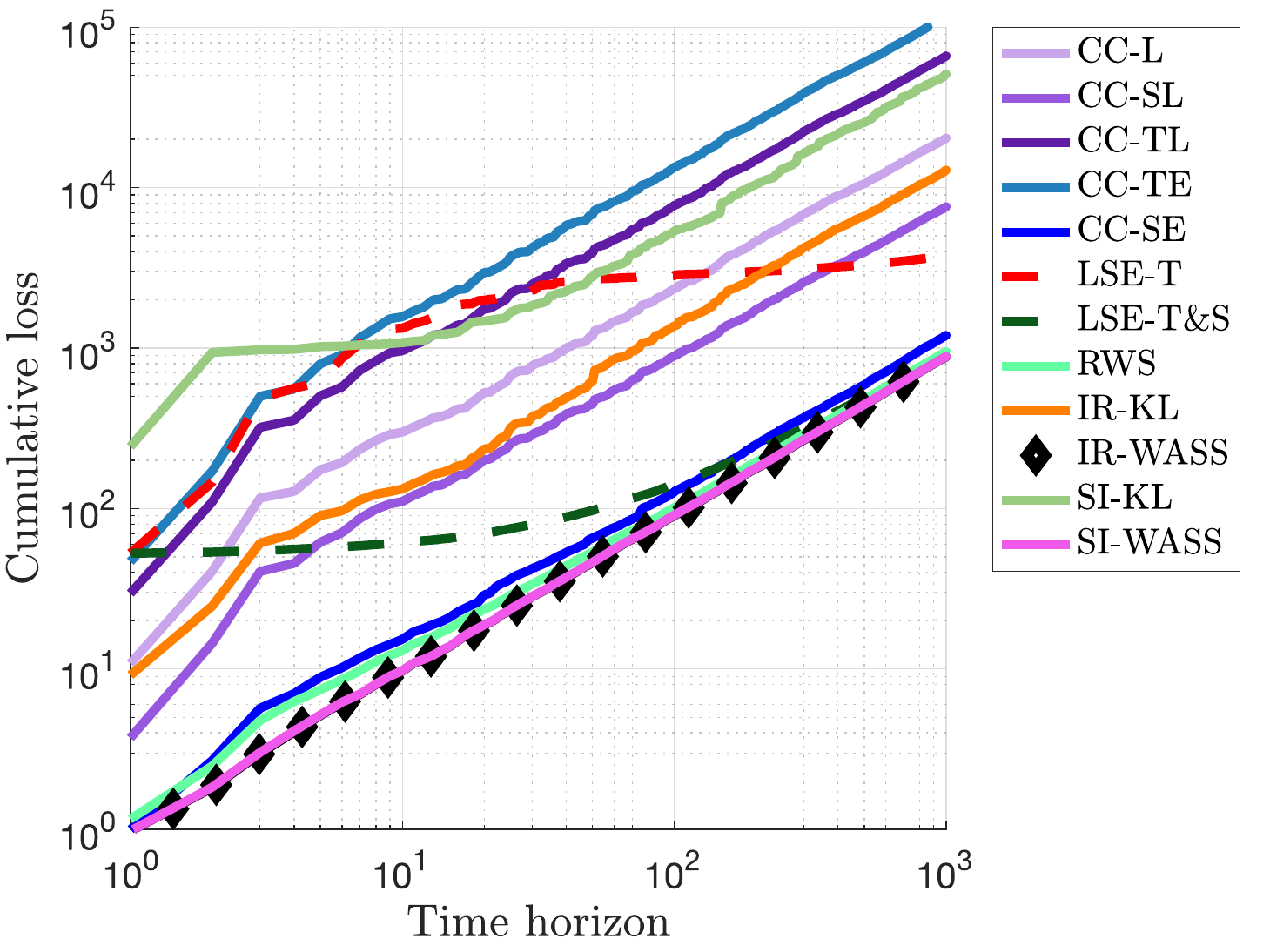}
    \vspace{-.2cm}
    \caption{Cumulative loss averaged over~100 runs, Uber$\&$Lyft.}
    \label{fig:cumloss_plot}
    \end{figure}
    
    Figure~\ref{fig:cumloss_plot} demonstrates how the average cumulative loss in~\eqref{eq:cumulative-loss} grows over time for the Uber$\&$Lyft dataset. 
    Figure~\ref{fig:cumloss_plot} shows that the loss of LSE-T$\&$S is initially constant at a high level, which highlights the discrepancy between the two domain distributions.
    The growth rate of LSE-T decays faster than that of other experts, and the time when LSE-T saturates indicates when the combined target domain data alone is sufficient to construct a single, competitive predictor without using any source domain data.

    \textbf{Concluding Remarks.}
    The theoretical and experimental results in this paper suggest that IR-WASS and SI-WASS are attractive schemes to generate a family of robust least squares experts. Moreover, the IR-WASS and SI-WASS experts are extremely easy to compute because it requires solving only a second-order cone or a linear semidefinite program. We observe that KL-type divergence schemes are less numerically stable due to the computation of the log-determinant and the inverse of a nearly singular covariance matrix $\covsa_{\rm T}$. Setting the parameters for KL-type divergence schemes is also harder due to the asymmetry of the divergence $\mathds D$. While this paper focuses solely on \textit{interpolating} schemes, it would also be interesting to explore \textit{extrapolating} schemes in future research.

    \section*{Acknowledgments}
    Material in this paper is based upon work supported by the Air Force Office of Scientific Research under award number FA9550-20-1-0397. Additional support is gratefully acknowledged from NSF grants 1915967, 1820942, 1838676, and also from the China Merchant Bank. Man-Chung Yue gratefully acknowledges the support by HKRGC under the Early Career Scheme Funding 25302420.
        
\bibliography{bibliography}
\bibliographystyle{icml2021}
\appendix
	\renewcommand\thesection{\Alph{section}}
	\renewcommand{\theequation}{A.\arabic{equation}}
	\renewcommand{\thefigure}{A.\arabic{figure}}
	\renewcommand{\thetable}{A.\arabic{table}}
\onecolumn
\section{Appendix}
\subsection{Proof of Section~\ref{sec:IR}}
\begin{proof}[Proof of Proposition~\ref{prop:kl_interpolation}]
Note that optimization problem~\eqref{eq:mean_cov_interpolation} constitutes an unbounded convex optimization problem when $\psi$ is the Kullback-Leibler-type divergence of Definition~\ref{def:divergence}.
Let $g(\mu, \Sigma) \Let \lambda \mathds{D}((\mu, \Sigma) \parallel (\msa_{\rm S}, \covsa_{\rm S})) + (1-\lambda) \mathds{D}((\mu, \Sigma) \parallel (\msa_{\rm T}, \covsa_{\rm T}))$,
then, the first order optimality condition reads
\begin{align*}
\nabla_{\mu}g(\mu, \Sigma) &= 2\lambda\covsa^{-1}_{\rm S} (\mu-\msa_{\rm S}) + 2(1-\lambda)\covsa^{-1}_{\rm T} (\mu-\msa_{\rm T}) = 0,\\
\nabla_{\Sigma}g(\mu, \Sigma) &= \lambda \covsa^{-1}_{\rm S} - \lambda \Sigma^{-1} + (1-\lambda) \covsa^{-1}_{\rm T} -( 1 - \lambda) \Sigma^{-1} = 0.
\end{align*}
One can then show $(\msa_\lambda, \covsa_\lambda)$ provided in statement of Proposition~\ref{prop:kl_interpolation} solves the system of equalities above.
\end{proof}

Below we prove Proposition~\ref{prop:grad_f_D}. In the proof of Proposition~\ref{prop:grad_f_D} and its auxiliary lemmas, Lemma~\ref{lemma:dual-KL} and Lemma~\ref{lemma:extreme-KL}, we omit the subscripts $\lambda$ and $\rho$ to avoid clutter.



\begin{lemma}[Dual problem] \label{lemma:dual-KL}
    Fix $(\msa, \covsa) \in \R^p \times \PD^p$ and $\rho \ge 0$. For any symmetric matrix $H \in \Sym^p$, the optimization problem
    \begin{subequations}
    \be \label{eq:KL-subproblem2}
    \left\{
	\begin{array}{cl}
	\Sup{\m, \cov } & \Tr{H (\cov+\m\m^\top)} \\
	\st & \Tr{\cov \covsa^{-1}} - \log\det (\cov \covsa^{-1}) - p + (\m - \msa)^\top \covsa^{-1} (\m - \msa) \leq \rho, \\
	& \cov \succ 0
	\end{array}
	\right.
	\ee
	admits the dual formulation
	\be \label{eq:KL-subproblem1}
	\left\{
	\begin{array}{cl}
	\inf &\dualvar (\rho - \msa^\top \covsa^{-1} \msa ) + \dualvar^2  \msa^\top \covsa^{-1} [\dualvar \covsa^{-1} - H]^{-1} \covsa^{-1} \msa  - \dualvar \log \det (I - \covsa^{\half} H \covsa^\half / \dualvar) \\
	\st & \dualvar \ge 0, \; \dualvar \covsa^{-1} \succ H.
	\end{array}
	\right.
	\ee
    \end{subequations}
\end{lemma}
\begin{proof}[Proof of Lemma~\ref{lemma:dual-KL}]
    For any $\m \in \R^p$ such that $(\m - \msa)^\top \covsa^{-1} (\m - \msa) \le \rho$, denote the set $\mc S_\m$ as
	\[
	\mc S_\m \Let \left\{
	\cov \in \PD^p : \Tr{\cov \covsa^{-1}} - \log\det \cov\leq \rho_\mu
	\right\},
	\]
	where $\rho_\m \in \R$ is defined as $\rho_{\m} \Let \rho + p - \log\det\covsa - (\m - \msa)^\top \covsa^{-1}(\m - \msa)$. Using these auxiliary notations, problem~\eqref{eq:KL-subproblem2} can be re-expressed as a nested program of the form
	\[
	\begin{array}{cl}
	\Sup{\m} & \m^\top H \m + \Sup{\cov \in \mc S_\m} ~ \Tr{H \cov} \\
	\st & (\m - \msa)^\top \covsa^{-1} (\m - \msa) \leq \rho,
	\end{array}
	\]
	where we emphasize that the constraint on $\m$ is redundant, but it is added to ensure the feasibility of the inner supremum over $\cov$ for every feasible value of $\m$ of the outer problem. We now proceed to reformulate the supremum subproblem over $\cov$.
	
	Assume momentarily that $H \neq 0$ and that $\m$ satisfies $(\m - \msa)^\top \covsa^{-1} (\m - \msa) < \rho$. In this case, one can verify that $\covsa$ is a Slater point of the convex set $\mc S_\m$. Using a duality argument, we find
	\begin{align*}
	\Sup{\cov \in \mc S_\m} ~ \Tr{H \cov} =& \Sup{\cov \succ 0} \Inf{\phi \ge 0} ~\Tr{H \cov} + \phi \big(\rho_\m - \Tr{\covsa^{-1} \cov} + \log\det \cov \big) \notag\\
	=& \Inf{\phi \ge 0} ~\left\{ \phi \rho_\m + \Sup{\cov \succ 0}~ \big\{ \Tr{(H - \phi \covsa^{-1})\cov}  + \phi \log \det \cov \big\} \right\},  \notag
	\end{align*}
	where the last equality follows from strong duality~\citep[Proposition~5.3.1]{ref:bertsekas2009convex}. 
	If $H - \phi \covsa^{-1} \not\prec 0$, then the inner supremum problem becomes unbounded. To see this, let $\sigma \in \R_+$ be the maximum eigenvalue of $H - \phi \covsa^{-1}$ with the corresponding eigenvector $v$, then the sequence $(\Sigma_k)_{k\in \mbb N}$ with $\Sigma_k = I + k vv^\top$ attains the asymptotic maximum objective value of $+\infty$. If $H - \phi \covsa^{-1} \prec 0$  then the inner supremum problem admits the unique optimal solution
	\be \label{eq:unique-cov}
	\cov\opt(\phi) = \phi (\phi \covsa^{-1} - H)^{-1},
	\ee
	which is obtained by solving the first-order optimality condition. By placing this optimal solution into the objective function and arranging terms, we have
	\be \label{eq:support-inner}
	\Sup{\cov \in \mc S_\m} ~ \Tr{H \cov} = \Inf{\substack{\phi \ge 0 \\ \phi \covsa^{-1} \succ H }}~ \phi \big( \rho - (\m - \msa)^\top \covsa^{-1} (\m - \msa) \big) - \phi \log \det (I - \covsa^\half H \covsa^\half /\phi).
	\ee
	We now argue that the above equality also holds when $\m$ is chosen such that $(\m - \msa)^\top \covsa^{-1} (\m - \msa) = \rho$. In this case, $\mc S_\m$ collapses into a singleton $\{\covsa\}$, and the left-hand side supremum problem attains the value $\Tr{H\covsa}$. The right-hand side infimum problem becomes
	\[
	    \Inf{\substack{\phi \ge 0 \\ \phi \covsa^{-1} \succ H }}~ - \phi \log \det (I - \covsa^\half H \covsa^\half /\phi).
	\]
	One can show using the l'Hopital rule that
	\[
	    \lim_{\phi \uparrow +\infty}~- \phi \log \det (I - \covsa^\half H \covsa^\half /\phi) = \Tr{H\covsa},
	\]
	which implies that the equality holds. Furthermore, when $H = 0$, the left-hand side of~\eqref{eq:support-inner} evaluates to 0, while the infimum problem on the right-hand side of~\eqref{eq:support-inner} also attains the optimal value of 0 asymptotically as $\phi$ decreases to 0. This implies that~\eqref{eq:support-inner} holds for all $H \in \mathbb{S}^p$ and for any $\m$ satisfying $(\m - \msa)^\top \covsa^{-1} (\m - \msa) \le \rho$. 
	
	The above line of argument shows that problem~\eqref{eq:KL-subproblem2} can now be expressed as the following maximin problem
	\[
	    \Sup{\m: (\m - \msa)^\top \covsa^{-1} (\m - \msa) \leq \rho} ~ \Inf{\substack{\phi \ge 0 \\ \phi \covsa^{-1} \succ H }}~ \m^\top H \m + \phi \big( \rho - (\m - \msa)^\top \covsa^{-1} (\m - \msa) \big) - \phi \log \det (I - \covsa^\half H \covsa^\half /\phi).
	\]
	For any $\phi\ge 0$ such that $\phi \covsa^{-1} \succ H$, the objective function is concave in $\m$. For any $\m$, the objective function is convex in $\phi$. Furthermore, the feasible set of $\mu$ is convex and compact, and the feasible set of $\phi$ is convex. As a consequence, we can apply Sion's minimax theorem~\cite{ref:sion1958minimax} to interchange the supremum and the infimum operators, and problem~\eqref{eq:KL-subproblem2} is equivalent to
	\[
	\Inf{\substack{\phi \ge 0 \\ \phi \covsa^{-1} \succ H }}~\left\{
	\begin{array}{l}
	\phi \rho - \phi \log \det  (I - \covsa^\half H \covsa^\half /\phi)  \\
	\hspace{2cm} + \Sup{\m:(\m - \msa)^\top \covsa^{-1} (\m - \msa) \leq \rho} ~ \m^\top H \m  - \phi (\m - \msa)^\top \covsa^{-1} (\m - \msa)
	\end{array}
	\right\}.
	\]
	For any $\phi$ which is feasible for the outer problem, the inner supremum problem is a convex quadratic optimization problem because $ \phi \covsa^{-1} \succ H$. Using a strong duality argument, the value of the inner supremum equals to the value of
	\begin{align*}
	    &\Inf{\nu \ge 0} ~ \left\{ \nu \rho - (\nu + \phi) \msa^\top \covsa^{-1} \msa + \Sup{\m}~ \m^\top (H - (\phi + \nu) \covsa^{-1}) \m + 2 (\nu + \phi) (\covsa^{-1} \msa)^\top \m \right\}\\
    =& \Inf{\nu \ge 0} ~ \nu \rho - (\nu + \phi) \msa^\top \covsa^{-1} \msa + (\nu + \phi)^2 (\covsa^{-1} \msa)^\top [(\phi + \nu) \covsa^{-1} - H]^{-1}  (\covsa^{-1} \msa ),
	\end{align*}
	where the equality follows from the fact that the unique optimal solution in the variable $\m$ is given by
	\be \label{eq:unique-mu}
	    (\phi + \nu) [ (\phi + \nu)\covsa^{-1} - H]^{-1}\covsa^{-1}\msa.
	\ee
	By combining two layers of infimum problem and using a change of variables $\dualvar \leftarrow \phi + \nu$,  problem~\eqref{eq:KL-subproblem2} can now be written as
	\be \label{eq:KL-subproblem3}
	\left\{
	\begin{array}{cl}
	\inf & \dualvar (\rho - \msa^\top \covsa^{-1} \msa ) +  \dualvar^2  \msa^\top \covsa^{-1} [\dualvar \covsa^{-1} - H]^{-1} \covsa^{-1} \msa  - \phi \log \det (I - \covsa^\half H \covsa^\half /\phi) \\
	\st & \phi \ge 0, \; \phi \covsa^{-1} \succ H, \; \dualvar - \phi \ge 0.
	\end{array}
	\right.
	\ee
	We now proceed to eliminate the multiplier $\phi$ from the above problem. To this end, rewrite the above optimization problem as
	\[
	\begin{array}{cl}
	\inf &\dualvar (\rho - \msa^\top \covsa^{-1} \msa ) + \dualvar^2  \msa^\top \covsa^{-1} [\dualvar \covsa^{-1} - H]^{-1} \covsa^{-1} \msa + g(\dualvar)\\
	\st & \dualvar \ge 0, \; \dualvar \covsa^{-1} \succ H,
	\end{array}
	\]
	where $g(\dualvar)$ is defined for every feasible value of $\dualvar$ as
	\be \label{eq:f-def}
	g(\dualvar) \Let \left\{
	\begin{array}{cl}
	\inf & - \phi \log \det (I - \covsa^{\half} H \covsa^\half / \phi) \\
	\st & \phi \ge 0, \; \phi \covsa^{-1} \succ H, \; \phi \le \dualvar.
	\end{array}
	\right.
	\ee
	Let $g_0 (\phi)$ denote the objective function of the above optimization, which is independent of $\dualvar$. Let $\sigma_1, \ldots, \sigma_p$ be the eigenvalues of $\covsa^\half H \covsa^\half$, we can write the function $g$ directly using the eigenvalues $\sigma_1, \ldots, \sigma_p$ as
	\[
	g_0(\phi) = -\phi \sum_{i = 1}^p \log (1 - \sigma_i/\phi).
	\]
	It is easy to verify by basic algebra manipulation that the gradient of $g_0$ satisfies
	\[
	\nabla g_0 (\phi) = \sum_{i=1}^p \left[ \log\left( \frac{\phi}{\phi - \sigma_i} \right) - \frac{\phi}{\phi - \sigma_i} \right] + p \leq 0,
	\]
	which implies that the value of $\phi$ that solves~\eqref{eq:f-def} is $\dualvar$, and thus $g (\dualvar) = - \dualvar \log \det (I - \covsa^{\half} H \covsa^\half / \dualvar)$. Substituting $\phi$ by $\dualvar$ in problem~\eqref{eq:KL-subproblem3} leads to the desired claim.
\end{proof}

\begin{lemma}[Optimal solution attaining $f(\beta)$] \label{lemma:extreme-KL}
    For any $(\msa, \covsa) \in \R^p \times \PD^p$, $\rho \in \R_{++}$ and $w \in \R^p$, $f(\beta)$ equals to the optimal value of the optimization problem
    \begin{subequations}
    \be \label{eq:KL-subproblem4}
    \left\{
	\begin{array}{cl}
	\Sup{\m, \cov \succ 0} & w^\top (\cov+\m\m^\top) w \\
	\st & \Tr{\cov \covsa^{-1}} - \log\det (\cov \covsa^{-1}) - p + (\m - \msa)^\top \covsa^{-1} (\m - \msa) \leq \rho,
	\end{array}
	\right.
	\ee
	which admits the unique optimal solution
    \be \label{eq:KL-mcov}
        \cov\opt =
        \dualvar\opt(\dualvar\opt \covsa^{-1} - ww^\top)^{-1}, \qquad \m\opt =\cov\opt \covsa^{-1} \msa,
    \ee
    with $\dualvar\opt > w^\top \covsa w$ being the unique solution of the nonlinear equation
    \begin{equation}\label{eq:KL-FOC}
        \rho = \frac{(w^\top \msa)^2 w^\top \covsa w}{(\dualvar - w^\top \covsa w)^2} + \frac{w^\top \covsa w}{\dualvar - w^\top \covsa w} + \log\Big( 1 - \frac{w^\top \covsa w}{\dualvar}\Big).
    \end{equation}
    Moreover, we have $\dualvar\opt \le w^\top \covsa w \big(1 + 2\rho + \sqrt{1 + 4 \rho (w^\top \msa)^2} \big)/(2\rho)$.
    \end{subequations}
\end{lemma}
\begin{proof}[Proof of Lemma~\ref{lemma:extreme-KL}]
    First, note that
    \begin{align*}
        f(\beta) & = \Sup{\QQ \in \mbb B} \EE_\QQ\left[(\beta^\top X - Y)^2 \right] = \Sup{\QQ \in \mbb B} \EE_\QQ \left[w^\top \xi \xi^\top w \right] = \Sup{ (\m, \cov) \in {\mbb U} } w^\top \left( \cov + \m\m^\top  \right) w ,
    \end{align*}
    which, by the definition of $\mbb U$ and definition~\eqref{def:KL}, equals to the optimal value of problem~\eqref{eq:KL-subproblem4}.

    From the duality result in Lemma~\ref{lemma:dual-KL}, problem~\eqref{eq:KL-subproblem4} is equivalent to
    \[
	\begin{array}{cl}
	\inf &\dualvar (\rho - \msa^\top \covsa^{-1} \msa ) + (\dualvar \covsa^{-1} \msa)^\top [\dualvar \covsa^{-1} - ww^\top]^{-1} ( \dualvar \covsa^{-1} \msa ) - \dualvar \log \det (I - \covsa^{\half} ww^\top \covsa^\half / \dualvar) \\
	\st & \dualvar \ge 0, \; \dualvar \covsa^{-1} \succ ww^\top.
	\end{array}
	\]
	Applying~\citet[Fact~2.16.3]{ref:bernstein2009matrix}, we have the equalities
	\begin{align*}
	     \det (I - \covsa^{\half} ww^\top \covsa^\half / \dualvar) &= 1 - w^\top \covsa w/\dualvar \\
	     (\dualvar \covsa^{-1} - ww^\top)^{-1} &= \dualvar^{-1} \covsa + \dualvar^{-2} \big( 1 - w^\top \covsa w/\dualvar \big)^{-1} \covsa w w^\top \covsa,
	\end{align*}
	and thus by some algebraic manipulations we can rewrite
	\be \label{eq:KL-extreme1}
	    f (\beta) =  \left\{
	\begin{array}{cl}
	\inf &\dualvar \rho  + \frac{\dualvar (w^\top \msa)^2}{\dualvar - w^\top \covsa w } - \dualvar \log \big( 1 - w^\top \covsa w/\dualvar \big) \\
	\st & \dualvar > w^\top \covsa w.
	\end{array}
	\right.
	\ee
	Let $f_0$ be the objective function of the above optimization problem. The gradient of $f_0$ satisfies
	\[
	    \nabla f_0(\dualvar) = \rho - \frac{(w^\top \msa)^2 w^\top \covsa w}{(\dualvar - w^\top \covsa w)^2} - \frac{w^\top \covsa w}{\dualvar - w^\top \covsa w} - \log\Big( 1 - \frac{w^\top \covsa w}{\dualvar}\Big).
	\]
	By the above expression of $\nabla f_0 (\dualvar)$ and the strict convexity of $f_0 (\dualvar)$, the value $\dualvar\opt$ that solves~\eqref{eq:KL-FOC} is also the unique minimizer of~\eqref{eq:KL-extreme1}. In other words, $f_0 (\kappa) = f(\beta)$.
	
	We now proceed to show that $(\m\opt, \cov\opt)$ defined as in~\eqref{eq:KL-mcov} is feasible and optimal. First, we prove feasibility of $(\m\opt, \cov\opt)$. By direct computation,
	\begin{subequations}
	\be \label{eq:KL-feasiblity1}
	    (\m\opt - \msa)^\top \covsa^{-1} (\m\opt - \msa) = \msa^\top (\covsa^{-1} \cov\opt - I) \covsa^{-1} (\cov\opt \covsa^{-1} - I) \msa = \frac{(\msa^\top w)^2 w^\top \covsa w}{(\dualvar\opt - w^\top \covsa w)^2}.
	\ee
	Moreover, because $\cov\opt \covsa^{-1} = I + (\dualvar\opt - w^\top \covsa w)^{-1} \covsa ww^\top$, we have
	\be \label{eq:KL-feasiblity2}
	    \Tr{\cov\opt \covsa^{-1}} - \log\det (\cov\opt \covsa^{-1}) - p = (\dualvar\opt - w^\top \covsa w)^{-1} w^\top \covsa w + \log \big(1 - \frac{w^\top \covsa w}{\dualvar\opt}\big).
	\ee
	\end{subequations}
	Combining~\eqref{eq:KL-feasiblity1} and~\eqref{eq:KL-feasiblity2}, we have
	\begin{align*}
	     \Tr{\cov\opt \covsa^{-1}} - \log\det (\cov\opt \covsa^{-1}) - p +  (\m\opt - \msa)^\top \covsa^{-1} (\m\opt - \msa) = \rho,
	\end{align*}
	where the first equality follows from the definition of $\mathds D$, and the second equality follows from the fact that $\dualvar\opt$ solves~\eqref{eq:KL-FOC}. This shows the feasibility of $(\m\opt, \cov\opt)$.
	
	Next, we prove the optimality of $(\m\opt, \cov\opt)$.
	Through a tedious computation, one can show that
	\begin{align*}
	    &w^\top (\cov\opt + (\m\opt)(\m\opt)^\top) w = w^\top (\cov\opt + \cov\opt \covsa^{-1} \msa \msa^\top \covsa^{-1} \cov\opt) w\\
	    =& w^\top \covsa w \Big(1 + \frac{w^\top \covsa w}{\dualvar\opt - w^\top \covsa w} \Big) + (\msa^\top w)^2 \Big( 1 + \frac{2 w^\top \covsa w}{\dualvar\opt - w^\top \covsa w} \Big) + \frac{(w^\top \msa)^2 (w^\top \covsa w)^2}{(\dualvar\opt - w^\top \covsa w)^2}\\
	    =& \frac{\dualvar\opt w^\top \covsa w}{\dualvar\opt - w^\top \covsa w} + \frac{(\dualvar\opt)^2 (\msa^\top w)^2}{(\dualvar\opt - w^\top \covsa w)^2} \\
	    =& \frac{\dualvar\opt w^\top \covsa w}{\dualvar\opt - w^\top \covsa w}  + \frac{\dualvar\opt (\msa^\top w)^2 w^\top \covsa w}{(\dualvar\opt - w^\top \covsa w)^2} + \frac{\dualvar\opt (\msa^\top w)^2}{\dualvar\opt - w^\top \covsa w} \\
	    =& \dualvar\opt \rho - \dualvar\opt \log \big( 1- \frac{w^\top \covsa w}{\dualvar\opt} \big)  + \frac{\dualvar\opt (\msa^\top w)^2}{\dualvar\opt - w^\top \covsa w} = f_0(\dualvar\opt) = f(\beta),
	\end{align*}
	where the antepenultimate equality follows from the fact that $\dualvar\opt$ solves~\eqref{eq:KL-FOC}, and the last equality holds because $\dualvar\opt$ is the minimizer of~\eqref{eq:KL-extreme1}. Therefore, $(\m\opt, \cov\opt)$ is optimal to problem~\eqref{eq:KL-subproblem4}. The uniqueness of~$(\m\opt, \cov\opt)$ now follows from the unique solution of $\cov$ and $\mu$ with respect to the dual variables from~\eqref{eq:unique-cov} and~\eqref{eq:unique-mu}, respectively.
	
	It now remains to show the upper bound on $\dualvar\opt$. Towards that end, we note that for any $\dualvar > w^\top \covsa w$, 
	\begin{align*}
	    0 &  = \rho - \frac{(w^\top \msa)^2 w^\top \covsa w}{(\dualvar\opt - w^\top \covsa w)^2} - \frac{w^\top \covsa w}{\dualvar\opt - w^\top \covsa w} - \log\Big( 1 - \frac{w^\top \covsa w}{\dualvar\opt}\Big)  > \rho - \frac{(w^\top \msa)^2 w^\top \covsa w}{(\dualvar\opt - w^\top \covsa w)^2} - \frac{w^\top \covsa w}{\dualvar\opt - w^\top \covsa w}.
	\end{align*}
	Solving the above quadratic inequality in the variable $\dualvar\opt - w^\top \covsa w$ yields the desired bound. This completes the proof.
\end{proof}

We are now ready to prove Proposition~\ref{prop:grad_f_D}.

\begin{proof}[Proof of Proposition~\ref{prop:grad_f_D}]
    The convexity of $f$ follows immediately by noting that it is the pointwise supremum of the family of convex functions $\EE_\QQ[(\beta^\top X - Y)^2]$ parametrized by $\QQ$.

    To prove the continuously differentiability and the formula for the gradient, recall the expression~\eqref{eq:KL-extreme1} for the function $f(\beta)$:
    \begin{equation}\label{opt:f}
        f (\beta)  = \left\{
	\begin{array}{cl}
	\inf &\dualvar \rho  + \frac{\dualvar (w^\top \msa)^2}{\dualvar - w^\top \covsa w } - \dualvar \log \big( 1 - w^\top \covsa w/\dualvar \big) \\
	\st & \dualvar > w^\top \covsa w.
	\end{array}
	\right.
    \end{equation}
Problem~\eqref{opt:f} has only one constraint. Therefore, LICQ (hence MFCQ) always holds, which implies that the Lagrange multiplier $\zeta_\beta$ of problem~\eqref{opt:f} is unique for any $\beta$. Also, it is easy to see that the constraint of problem~\eqref{opt:f} is never binding. So, $\zeta_\beta = 0$ for any $\beta$. The Lagrangian function $L_{\beta}: \R \times \R \rightarrow \R $ is given by
\begin{equation*}
L_{\beta} (\dualvar, \zeta) = \rho \dualvar + \frac{\omega_2 \dualvar}{\dualvar - \omega_1} - \dualvar \log\left(1- \frac{\omega_1}{\dualvar} \right) + \zeta (\omega_1 - \dualvar),
\end{equation*}
where $\omega_1 =  w^\top \covsa w$ and $\omega_2 =  (w^\top \msa)^2$.
The first derivative with respect to $\dualvar$ is 
\begin{equation*}
\frac{\mathrm d L_\beta}{\mathrm d \dualvar}(\dualvar , \zeta) = \rho - \frac{\omega_1 \omega_2}{(\dualvar- \omega_1)^2 } - \log\left( 1 - \frac{\omega_1}{\dualvar} \right) - \frac{\omega_1}{\dualvar - \omega_1} - \zeta.
\end{equation*}
The second derivative with respect to $\dualvar$ is
\begin{equation*}
\frac{\mathrm d^2 L_\beta}{\mathrm d\dualvar^2}(\dualvar , \zeta) = \frac{\omega_1}{(\dualvar - \omega_1)^3} \left( 2\omega_2 + \frac{\omega_1}{\dualvar}(\dualvar - \omega_1) \right) .
\end{equation*}
From the proof of Lemma~\ref{lemma:extreme-KL}, we have that the minimizer $\dualvar_\beta$ of problem~\eqref{opt:f} is precisely the $ \dualvar\opt$ defined by equation~\eqref{eq:KL-FOC} (below we write $\dualvar_\beta$ instead of $ \dualvar\opt$ to emphasize and keep track of the dependence on $\beta$). Therefore, for any $\beta$, the minimizer $\dualvar_\beta$ exists and is unique.
So, there exists some constant $\eta_{\beta} > 0$ such that
\begin{equation*}
\frac{\mathrm d^2 L_{\beta}}{\mathrm d\dualvar^2}(\dualvar_{\beta} , \zeta_{\beta}) \ge \eta_{\beta} >0.
\end{equation*}
Therefore, for any $\beta$, the strong second order condition at $\dualvar_{\beta}$ holds (see~\citet[Definition 6.2]{still2018lectures}). By \citet[Theorem 6.7]{still2018lectures},
\begin{equation}\label{eq:f_grad}
\nabla f (\beta) = \nabla_\beta L_\beta (\dualvar_\beta , \zeta_\beta) = \nabla_\beta L_\beta (\dualvar_\beta , 0)\quad \forall \beta\in \R^d.
\end{equation}
Then we compute
\begin{align*}
\nabla_w L_{\beta} (\dualvar, \zeta) & = \nabla_w \left[ \frac{\dualvar (w^\top \msa)^2}{\dualvar - w^\top \covsa w } - \dualvar \log \left( 1 - \frac{w^\top \covsa w}{\dualvar}  \right) + \zeta (w^\top \covsa w - \dualvar) \right] \\
& = \frac{2\dualvar \omega_2}{(\dualvar - \omega_1)^2} \covsa w + \frac{2\dualvar}{(\dualvar - \omega_1)} \msa \msa^\top w + \frac{2\dualvar}{(\dualvar - \omega_1)} \covsa w + 2\zeta \covsa w.
\end{align*}
Hence,
\begin{align*}
&\, \nabla_\beta L_\beta (\dualvar, \zeta) = \frac{d w}{d\beta}^\top \cdot \nabla_w L_{\beta} (\dualvar, \zeta) = [I_d \ \mathbf{0}_d ] \cdot \nabla_w L_{\beta} (\dualvar, \zeta),
\end{align*}
which, when combined with \eqref{eq:f_grad}, yields the desired gradient formula
\begin{equation*}
\nabla f (\beta) = \frac{2\dualvar_\beta \left(  \omega_2  \covsa w \!+\! (\dualvar_\beta \! - \!\omega_1 ) (\covsa \! +\! \msa \msa^\top) w \right)_{1:d}}{(\dualvar_\beta  - \omega_1 )^2}.
\end{equation*}
By \citet[Theorem 6.5]{still2018lectures}, the function $\beta \mapsto \dualvar_\beta$ is locally Lipschitz continuous, \ie, for any $\beta\in \R^d$, there exists $c_\beta,\epsilon_\beta > 0$ such that if $\norm{\beta' - \beta}_2 \le \epsilon_\beta$, then
\begin{equation*}\label{ineq:gamma_Lip}
|\dualvar_{\beta'} - \dualvar_{\beta}| \le c_\beta \norm{ \beta' - \beta }_2.
\end{equation*}
Note that $\omega_1$ and $\omega_2 $ are both locally Lipschitz continuous in $\beta$. Also, it is easy to see that $\dualvar_\beta > \omega_1 $ for any $\beta$. Thus, $\nabla f (\beta)$ is locally Lipschitz continuous in $\beta$.
\end{proof}

\begin{proof}[Proof of~\ref{prop:wass_interpolation}]
Noting that problem~\eqref{eq:mean_cov_interpolation} is the barycenter problem between two Gaussian distributions with respect to the Wasserstein distance, the proof then directly follows from \citet[\S6.2]{agueh2011barycenters} and \citet[Example~1.7]{ref:mccann1997convexity}.
\end{proof}

\begin{proof}[Proof of Proposition~\ref{prop:IR_W}]
    Again we omit the subscripts $\lambda$ and $\rho$.
    Reminding that $\xi = (X, Y)$, we find
    \begin{equation}\label{eq:wass_f}
        \begin{split}
        &\Sup{\QQ \in \mbb B} \EE_\QQ[(\beta^\top X - Y)^2] = \Sup{\QQ \in \mbb B} \EE_\QQ[(w^\top \xi)^2] \\
        =& \left\{
	\begin{array}{cl}
	\inf & \dualvar \big(\rho - \|\msa\|_2^2 -  \Tr{\covsa} \big) + z + \Tr{Z} \\
	\st & \dualvar \in \R_+, \; z \in \R_+, \; Z \in \PSD^p \\
	& \begin{bmatrix} \dualvar I - ww^\top & \dualvar \covsa^\half \\ \dualvar \covsa^\half & Z \end{bmatrix} \succeq 0, \; \begin{bmatrix} \dualvar I - ww^\top & \dualvar \msa \\ \dualvar \msa^\top & z \end{bmatrix} \succeq 0
	\end{array}
	\right. \\
	=&\left\{
	    \begin{array}{cl}
	        \inf & \dualvar \big(\rho - \|\msa\|_2^2 -  \Tr{\covsa} \big) + \dualvar^2 \msa^\top (\dualvar I - ww^\top)^{-1} \msa + \dualvar^2 \Tr{\covsa (\dualvar I - ww^\top)^{-1}} \\
	        \st & \dualvar \ge \| w \|_2^2 ,
	    \end{array}
	\right.
    \end{split}
    \end{equation}
    where the second equality follows from \citet[Lemma 2]{ref:kuhn2019wasserstein}. By applying~\citet[Fact~2.16.3]{ref:bernstein2009matrix}, we find
	\begin{equation}\label{eq:sherman_morrison}
	     (\dualvar I - ww^\top)^{-1} = \dualvar^{-1} I + \dualvar^{-2} \big( 1 - \|w\|_2^2/\dualvar \big)^{-1}  w w^\top .
	\end{equation}
	Combining \eqref{eq:wass_f} and \eqref{eq:sherman_morrison}, we get
	\[
	    \Sup{\QQ \in \mbb B} \EE_\QQ[(\beta^\top X - Y)^2] =
	    \left\{
	        \begin{array}{cl}
	            \inf & \dualvar \rho + \dualvar  w^\top (\covsa + \msa \msa^\top) w / (\dualvar - \| w \|_2^2 )\\
	            \st & \dualvar \ge \| w\|_2^2.
	        \end{array}
	    \right.
	\]
	One can verify through the first-order optimality condition that the optimal solution $\dualvar\opt$ is
	\[
	    \dualvar\opt = \| w \|_2 \left( \| w \|_2 + \sqrt{\frac{w^\top (\covsa + \msa \msa^\top) w }{\rho}} \right),
	\]
	and by replacing this value $\dualvar\opt$ into the objective function, we find
	\[
	    \Sup{\QQ \in \mbb B} \EE_\QQ[(\beta^\top X - Y)^2] = \big( \sqrt{w^\top (\covsa + \msa\msa^\top) w} + \sqrt{\rho}\|w\|_2 \big)^2,
	\]
	which then completes the proof.
\end{proof}
\newpage
\subsection{Proof of Section~\ref{sec:SI}}

\begin{lemma}[Compactness] \label{lemma:D-set}
For $k\in \{\rm S, \rm T\}$, the set
\[\mbb V_k = \{(\mu, M) \in \R^p\times \mbb S^p_{++} : M- \mu \mu^\top \in \mbb S_{++}^p, \mathds{D}((\mu, M-\mu\mu^\top) \parallel (\msa_k, \covsa_k)) \leq \rho_{k} \} \] 
is convex and compact. Furthermore, the set
\[\mbb V \Let \{(\m, M) \in \R^p\times \mbb S^p_{++} :(\m, M- \m\m^\top) \in \mbb U_{\rho_{\rm S}, \rho_{\rm T}}\} \]
is also convex and compact.
\end{lemma}

\begin{proof}[Proof of Lemma~\ref{lemma:D-set}]
	For any $(\mu, M) \in \R^p\times \mbb S^p_{++} $ such that $M - \m\m^\top \in \PD^p$, we find
	\begin{align}
		&\mathds D\big((\m, M-\m\m^\top) \parallel (\msa_k, \covsa_k)\big) \notag\\
		=& (\m - \msa_k)^\top\covsa^{-1}_k (\m - \msa_k) +\Tr{(M - \m\m^\top) \covsa^{-1}} - \log\det ((M - \m\m^\top) \covsa_k^{-1}) - p \notag\\
		=& \msa_k^\top \covsa_k^{-1} \msa_k - 2 \msa_k^\top \covsa_k^{-1} \m+\Tr{M\covsa_k^{-1}} - \log\det (M\covsa_k^{-1}) - \log(1- \m^\top M^{-1} \m) - p \label{eq:divergence_mu_M},
	\end{align}
	where in the last expression, we have used the determinant formula~\citep[Fact~2.16.3]{ref:bernstein2009matrix} to rewrite
	\[
	    \det(M - \m\m^\top) = (1 - \m^\top M^{-1} \m) \det M.
	\]

	Because $M - \m\m^\top \in \PD^p$, one can show that $1 - \m^\top M^{-1} \m > 0$ by invoking the Schur complement, and as such, the logarithm term in the last expression is well-defined. Moreover, we can write 
	\begin{align}
	\mbb V_k = \left\{(\m, M) :
	\begin{array}{l}
	(\m, M) \in \R^p \times \PD^p,~M - \m\m^\top \in \PD^p,~\exists t \in \R_+: \\
	\msa_k^\top \covsa_k^{-1} \msa_k - 2 \msa_k^\top \covsa_k^{-1} \m+\Tr{M\covsa_k^{-1}} - \log\det (M\covsa_k^{-1}) - \log(1- t) - p \leq \rho \\
	\begin{bmatrix} M & \m \\ \m^\top & t \end{bmatrix} \succeq 0	\end{array}
	\right\}, \label{eq:D-refor}
	\end{align}
	which is a convex set. Notice that by Schur complement, the semidefinite constraint is equivalent to $t \ge \m^\top M^{-1} \m$.
	
	Next, we show that $\mbb V_k$ is compact. Denote by $\mbb U_k =  \{ (\m, \cov)\in \R^p\times \PSD^p:  \mathds{D}( (\m, \cov)\! \parallel\! (\msa_k, \covsa_k) )\le \rho_k \}$. Then, it is easy to see that $\mbb V_k$ is the image of $\mbb U_k$ under the continuous mapping $(\m, \cov) \mapsto (\m, \cov + \m\m^\top)$. Therefore, it suffices to prove the compactness of $\mbb U_k$. Towards that end, we note that 
	\[ {\mathds D} \big( (\m, \cov) \parallel (\msa_k, \covsa_k) \big) =(\msa_k - \m)^\top\covsa_k^{-1} (\msa_k - \m) +  \Tr{\cov \covsa_k^{-1}} - \log\det (\cov \covsa_k^{-1}) - p  \]
	is a continuous and coercive function in $(\m, \cov)$. Thus, as a level set of ${\mathds D} \big( (\m, \cov) \parallel (\msa_k, \covsa_k) \big)$, $\mbb U_k$ is closed and bounded, and hence compact.
	
	To prove the last claim, by the definitions of $\mbb V$ and $\mbb U_{\rho_{\rm S}, \rho_{\rm T}}$ we write
	\begin{align}
	    &\mbb V = \{(\m, M) \in \R^p\times \mbb S^p_{++} :(\m, M- \m\m^\top) \in \mbb U_{\rho_{\rm S}, \rho_{\rm T}}\}\notag\\
	    = &\{(\m, M) \in \R^p\times \mbb S^p_{++} : (\m, M) \in \mbb V_{\rm S}\} \cap \{(\m, M) \in \R^p\times \mbb S^p_{++}: (\m, M) \in \mbb V_{\rm T} \}\cap \{(\m, M) \in \R^p\times \mbb S^p_{++}: M\succeq \eps I\} \label{eq:V_intersect}.
	\end{align}
	The convexity of $\{(\m, M) \in \R^p\times \mbb S^p_{++} :(\m, M- \m\m^\top) \in \mbb U_{\rho_{\rm S}, \rho_{\rm T}}\}$ then follows from the convexity of the three sets in~\eqref{eq:V_intersect}.
	Furthermore, from the first part of the proof, we know that both $\{(\m, M) \in \R^p\times \mbb S^p_{++} : (\m, M) \in \mbb V_{\rm S}\}$ and $\{(\m, M) \in \R^p\times \mbb S^p_{++}: (\m, M) \in \mbb V_{\rm T} \}$ are compact sets, so is their intersection. Also, the last set $\{(\m, M) \in \R^p\times \mbb S^p_{++}: M\succeq \eps I\}$ in \eqref{eq:V_intersect} is closed. Since any closed subset of a compact set is again compact, we conclude that $\mbb V$ is compact. This completes the proof.

\end{proof}

\begin{proof}[Proof of Theorem~\ref{thm:ls-kl}]
    As $\xi = (X, Y)$, we can rewrite
    \begin{subequations}
    \begin{align}
        & \Min{\beta \in \R^d} \Sup{\QQ \in \mbb B_{\rho_{\rm S}, \rho_{\rm T}}} \EE_\QQ[(\beta^\top X - Y)^2] \\
        =&\Min{\beta \in \R^d} 
        \Sup{\QQ \in \mbb B_{\rho_{\rm S}, \rho_{\rm T}}} \!\begin{bmatrix} \beta \\ -1 \end{bmatrix}^\top \EE_\QQ[\xi \xi^\top] \begin{bmatrix} \beta \\ -1 \end{bmatrix}
        \\
        = &\Min{\beta \in \R^d} \Sup{(\m, M- \m\m^\top) \in \mbb U_{\rho_{\rm S}, \rho_{\rm T}}} \begin{bmatrix} \beta \\ -1 \end{bmatrix}^\top M \begin{bmatrix} \beta \\ -1 \end{bmatrix} \notag\\
        = &\Min{\beta \in \R^d} \Sup{(\m, M) \in \mbb V} \begin{bmatrix} \beta \\ -1 \end{bmatrix}^\top M \begin{bmatrix} \beta \\ -1 \end{bmatrix} \notag\\
        = &\Sup{(\m, M) \in \mbb V}\Min{\beta \in \R^d}  \begin{bmatrix} \beta \\ -1 \end{bmatrix}^\top M \begin{bmatrix} \beta \\ -1 \end{bmatrix} \label{eq:aux-1}\\
        =& \Sup{(\m, M) \in \mbb V}~M_{YY} - M_{XY}^\top M_{XX}^{-1} M_{XY} \label{eq:aux-2}
\end{align}
\end{subequations}
where~\eqref{eq:aux-1} follows from the Sion's minimax theorem, which holds because the objective function is convex in $\beta$, concave in $M$, and Lemma~\ref{lemma:D-set}. Equation~\eqref{eq:aux-2} exploits the unique optimal solution in $\beta$ as $\beta\opt = M_{XX}^{-1} M_{XY}$, in which the matrix inverse is well defined because $M \succ 0$ for any feasible $M$.

Finally, after an application of the Schur complement reformulation to~\eqref{eq:aux-2}, the nonlinear semidefinite program in the theorem statement follows from representations~\eqref{eq:D-refor} and \eqref{eq:V_intersect}. This completes the proof.
\end{proof}

\begin{proof}[Proof of Proposition~\ref{prop:minimum_radius}]
It is well-known that the space of probability measures equipped with the Wasserstein distance $W_2 $ is a geodesic metric space (see \citet[Section 7]{ref:villani2008optimal} for example), meaning that for any two probability distributions $\mc N_0$ and $\mc N_1$, there exists a constant-speed geodesic curve $[0,1] \ni a\mapsto \mc N_a$ satisfying 
\[ W_2 ( \mc N_a, \mc N_{a'} ) = |a - a'| W_2 ( \mc N_0, \mc N_1 ) \quad\forall a,a'\in [0,1].\]

The claim follows trivially if $W_2 ( \mc N_{\rm S}, \mc N_{\rm T} ) \le \sqrt{\rho_{\rm S}}$. Therefore, we assume $W_2 ( \mc N_{\rm S}, \mc N_{\rm T} ) >\sqrt{\rho_{\rm S}}$.

Consider the the geodesic $\mc N_t$ from $\mc N_0 = \mc N_{\rm S}$ to $\mc N_1 = \mc N_{\rm T}$. Also, denote by $\mbb U_k =  \{ (\m, \cov)\in \R^p\times \PSD^p:  \mathds{D}( (\m, \cov)\! \parallel\! (\msa_k, \covsa_k) )\le \rho_k \}$ for $k\in \{\rm S, \rm T\} $. Then, $\mbb U_{\rm S}$ and $\mbb U_{\rm T} $ has empty intersection if and only if 
\[ W_2 ( \mc N_a , \mc N_{\rm S} ) \le \sqrt{\rho_{\rm S}} \Longrightarrow W_2 ( \mc N_a , \mc N_{\rm T} ) > \sqrt{\rho_{\rm T}} \quad\forall a\in [0,1],\]
which is in turn equivalent to 
\[ a W_2 ( \mc N_{\rm T} , \mc N_{\rm S} ) \le \sqrt{\rho_{\rm S}} \Longrightarrow (1-a) W_2 ( \mc N_{\rm T} , \mc N_{\rm S} ) \le \sqrt{\rho_{\rm T}} \quad \forall a\in[0,1].\]
Picking $a = \frac{\sqrt{\rho_{\rm S}}}{W_2 ( \mc N_{\rm T} , \mc N_{\rm S} )} \in (0,1)$, then we have
\begin{align*}
    \left( 1 - \frac{\sqrt{\rho_{\rm S}}}{W_2 ( \mc N_{\rm T} , \mc N_{\rm S} )} \right) W_2 ( \mc N_{\rm T} , \mc N_{\rm S} ) \le \sqrt{\rho_{\rm T}}.
\end{align*}
The above inequality can be rewritten as
\[ W_2 ( \mc N_{\rm T} , \mc N_{\rm S} ) \le \sqrt{\rho_{\rm S}} + \sqrt{\rho_{\rm T}}, \]
which contradicts with our supposition
\[\rho_{\rm T}\geq \left(\sqrt{\mathds  W((\msa_{\rm S}, \covsa_{\rm S}) \parallel (\msa_{\rm T}, \covsa_{\rm T}))} - \sqrt{\rho_{\rm S}}\right)^{2}.\]
Thus, $\mbb U_{\rm S}$ and $\mbb U_{\rm T}$ has non-empty intersection.
\end{proof}

\begin{proof}[Proof of Theorem~\ref{thm:ls-w}]
    As $\xi = (X, Y)$, we can rewrite
    \begin{subequations}
    \begin{align}
    &\Min{\beta \in \R^d} \Sup{\QQ \in \mbb B_{\rho_{\rm S}, \rho_{\rm T}}(\Pnom)} \EE_\QQ[(\beta^\top X - Y)^2] \\
        = & \Min{\beta \in \R^d} \Sup{(\m, M - \mu \mu^\top) \in \mbb U_{\rho_{\rm S}, \rho_{\rm T}}} \begin{bmatrix} \beta \\ -1 \end{bmatrix}^\top M \begin{bmatrix} \beta \\ -1 \end{bmatrix} \notag\\
        = & \Sup{(\m, M - \m\m^\top) \in \mbb U_{\rho_{\rm S}, \rho_{\rm T}}}\Min{\beta \in \R^d}  \begin{bmatrix} \beta \\ -1 \end{bmatrix}^\top M \begin{bmatrix} \beta \\ -1 \end{bmatrix} \label{eq:auxg-1}\\
        = & \Sup{(\m, M - \m\m^\top) \in \mbb U_{\rho_{\rm S}, \rho_{\rm T}}}~M_{YY} - M_{XY}^\top M_{XX}^{-1} M_{XY} 
        \label{eq:auxg-3}
\end{align}
\end{subequations}
where~\eqref{eq:auxg-1} follows from the Sion's minimax theorem, which holds because the objective function is convex in $\beta$, concave in $M$, and the set $\mbb U_{\rho_{\rm S}, \rho_{\rm T}}$ is compact~~\citep[Lemma~A.6]{ref:abadeh2018wasserstein}. Equation~\eqref{eq:auxg-3} exploits the unique optimal solution in $\beta$ as $\beta\opt = M_{XX}^{-1} M_{XY}$, in which the matrix inverse is well defined because $M - \m\m^\top \succeq \eps I$ for any feasible $M$.
\end{proof}

\section{Additional Numerical Results}
In the following the details of the datasets used in Section~\ref{sec:numerical} are presented.
    \begin{itemize}[leftmargin = 3mm]
      \item \textbf{Uber$\&$Lyft\footnote{Available publicly at~\url{https://www.kaggle.com/brllrb/uber-and-lyft-dataset-boston-ma}}} has~$N_{\rm S} = 5000$ instances in the source domain and~5000 available samples in the target domain. 
      \item \textbf{US Births~(2018)\footnote{Available publicly at~\url{https://www.kaggle.com/des137/us-births-2018}}} has~$N_{\rm S} = 5172$ samples in the source domain and~4828 available samples in the target domain.
       \item \textbf{Life Expectancy{\footnote{Available publicly at~\url{https://www.kaggle.com/kumarajarshi/life-expectancy-who}}}} has~$N_{\rm S} = 1407$ instances in the source domain and~242 available samples in the target domain.
     \item \textbf{House Prices in King County\footnote{Available publicly at~\url{https://www.kaggle.com/c/house-prices-advanced-regression-techniques/data}}} has~$N_{\rm S} = 543$ instances in the source domain and~334 available samples in the target domain.
         \item \textbf{California Housing Prices\footnote{The modified version that we use is available publicly at~\url{https://www.kaggle.com/camnugent/california-housing-prices} and the original dataset is available publicly at~\url{https://www.dcc.fc.up.pt/~ltorgo/Regression/cal_housing.html}}} has $N_{\rm S} = 9034$ instances in the source domain, and~6496 available instances in the target domain.
    \end{itemize}
    \begin{figure}
     \centering
     \begin{subfigure}[b]{0.45\textwidth}
         \centering
         \includegraphics[width=\textwidth]{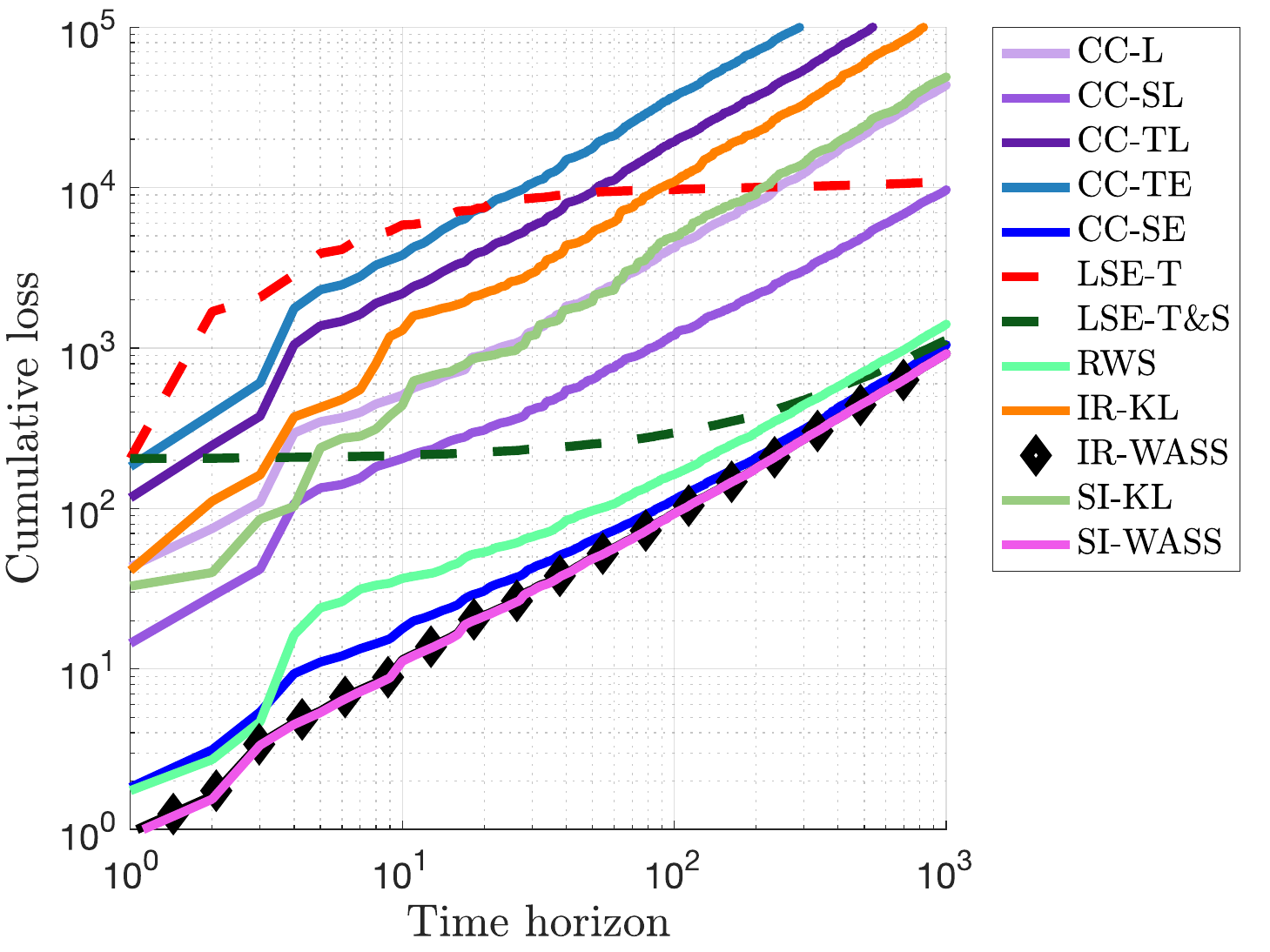}
         \caption{US Births (2018)}
         \label{fig:us_births}
     \end{subfigure}\hfill
     \begin{subfigure}[b]{0.45\textwidth}
         \centering
         \includegraphics[width=\textwidth]{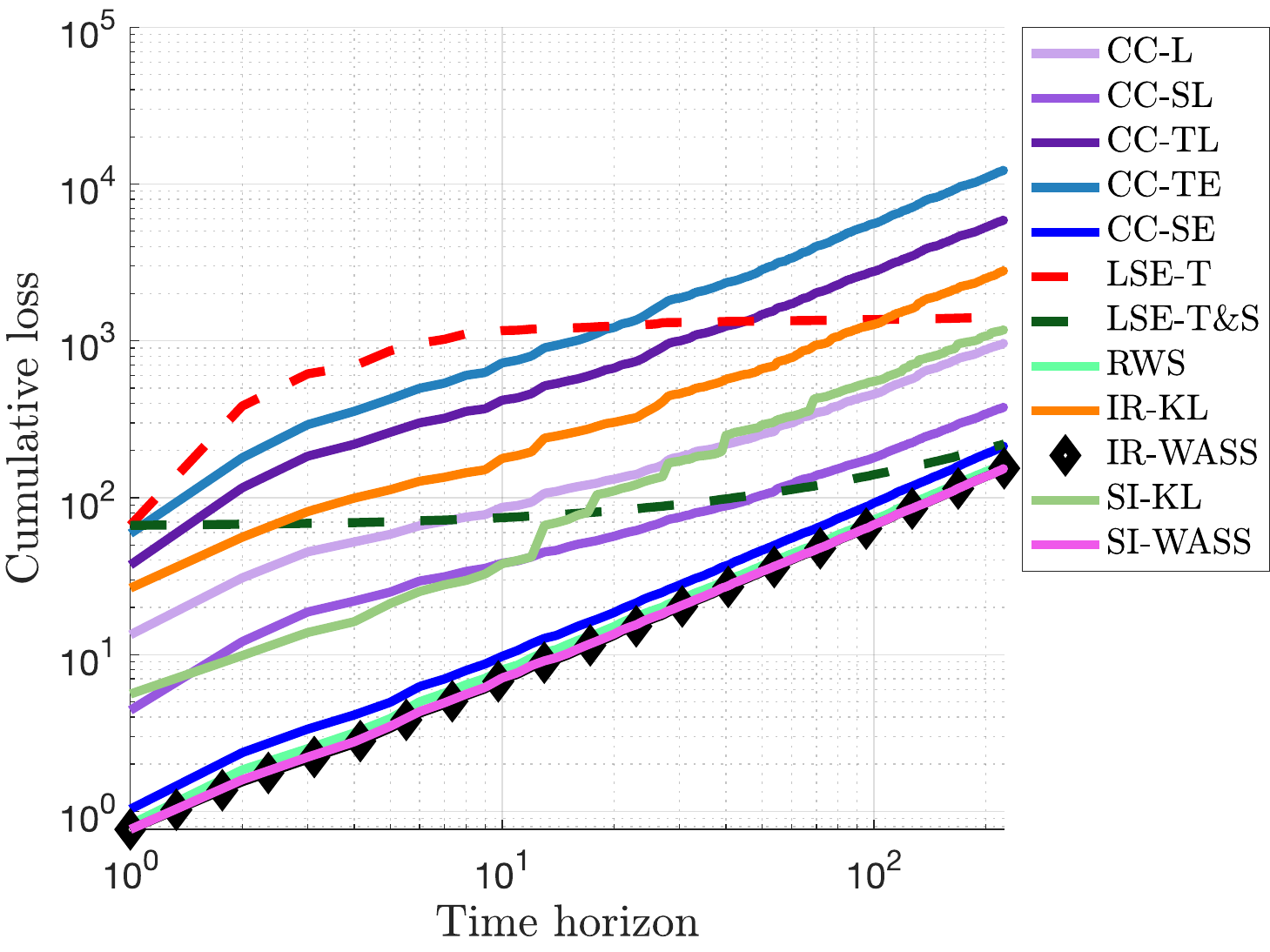}
         \caption{Life Expectancy}
         \label{fig:life_exp}
     \end{subfigure}\hfill
     \begin{subfigure}[b]{0.45\textwidth}
         \centering
         \includegraphics[width=\textwidth]{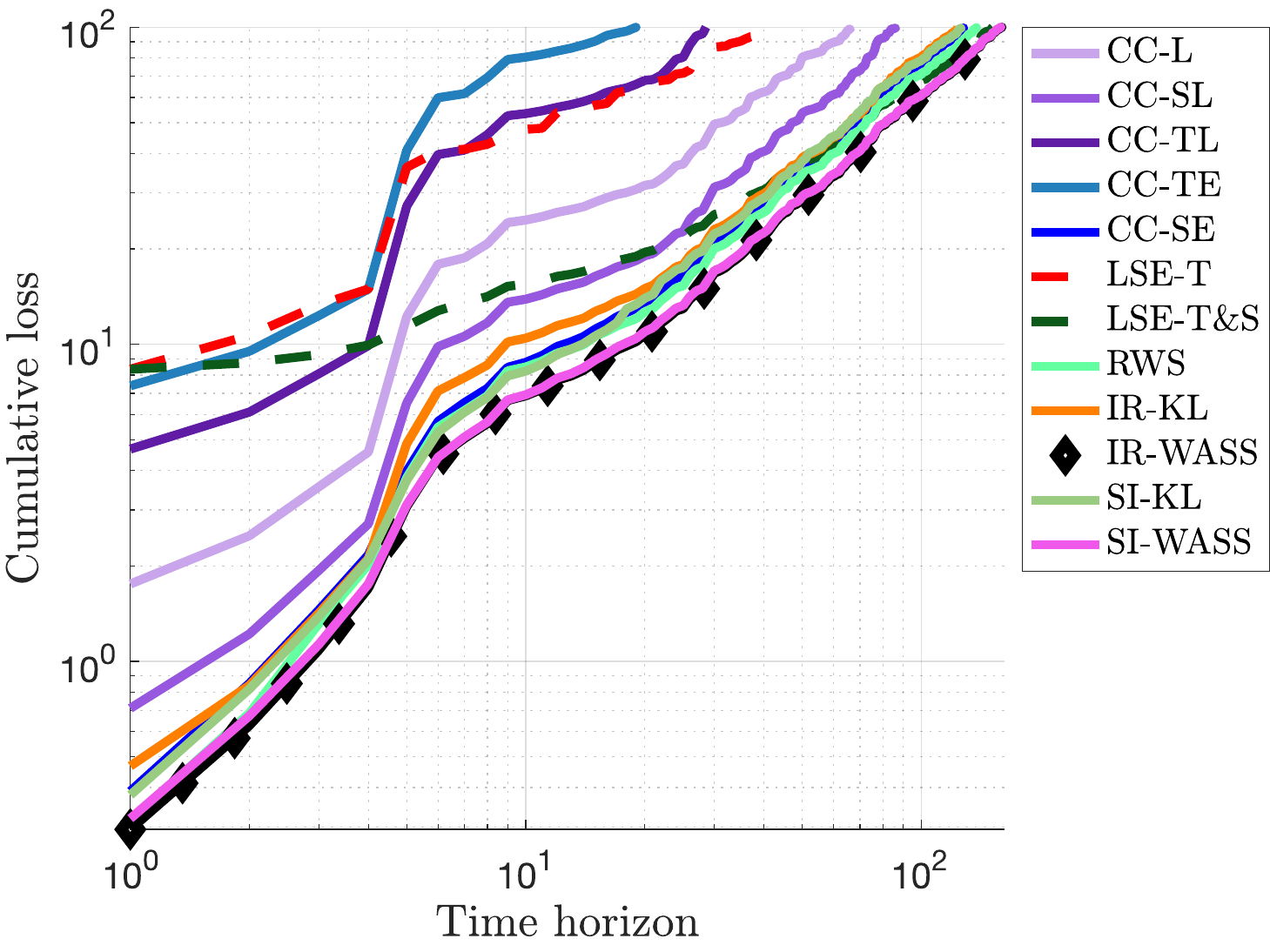}
         \caption{House Prices in KC}
         \label{fig:houses}
     \end{subfigure}\hfill
     \begin{subfigure}[b]{0.45\textwidth}
         \centering
         \includegraphics[width=\textwidth]{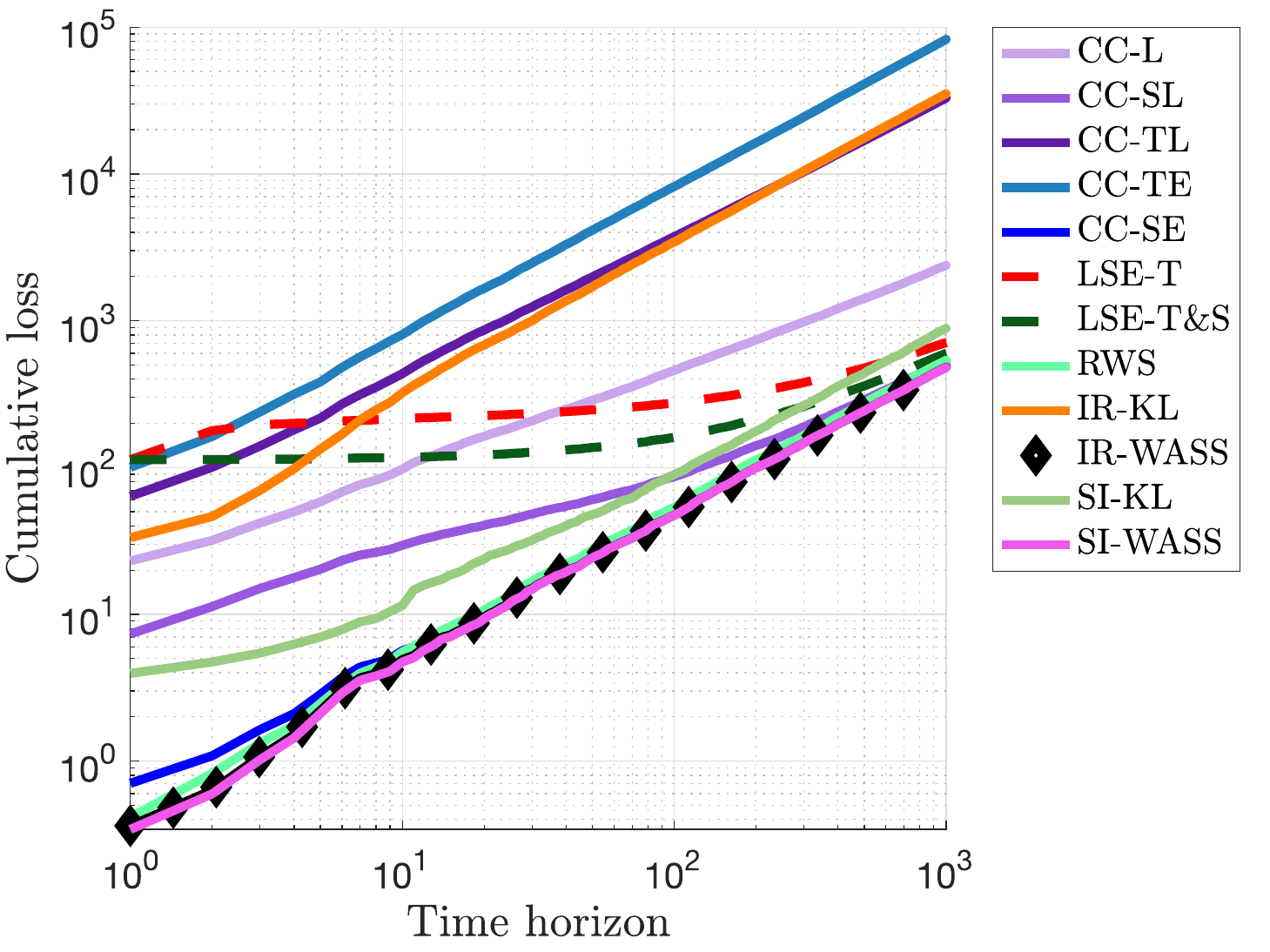}
         \caption{California Housing}
         \label{fig:california_housing}
     \end{subfigure}\hfill
        \caption{Cumulative loss averaged over 100 runs on logarithmic scale}
        \label{fig:cum_loss_all}
\end{figure}
Figure~\ref{fig:cum_loss_all} demonstrates how the average cumulative loss in~\eqref{eq:cumulative-loss} grows over time for the US Births (2018), Life Expectancy, House Prices in KC and California Housing datasets. 
The results suggest that the IR-WASS and SI-WASS experts perform favorably over the competitors in that their cumulative loss at each time step is lower than that of most other competitors.  

\end{document}